\pdfoutput=1 

\documentclass{article}
\usepackage{arxiv}

\renewcommand{\shorttitle}{Uncertainty-Aware Dist-to-Dist Flow Matching for Scientific Image Generation}

\usepackage[bottom]{footmisc}

\makeatletter
\renewenvironment{table}{\@float{table}}{\end@float}
\makeatother

\makeatletter
\providecommand\@LN[2]{}
\providecommand\@LN@col[1]{}
\makeatother

\usepackage[dvipsnames]{xcolor}
\usepackage{graphicx}
\usepackage{amsmath}
\usepackage{amssymb}
\usepackage{booktabs}
\usepackage[numbers,sort&compress]{natbib}
\usepackage{caption}
\usepackage[hyphens]{url}

\usepackage[utf8]{inputenc} 
\usepackage[T1]{fontenc}    
\usepackage{amsfonts}       
\usepackage{nicefrac}       
\usepackage{microtype}      
\usepackage{colortbl}
\usepackage{float}          

\usepackage{amsmath,amsfonts,bm}

\def\eqref#1{equation~\ref{#1}}

\def\1{\bm{1}}

\def\vW{{\bm{W}}}

\def\vzero{{\bm{0}}}

\def\vb{{\bm{b}}}

\def\vu{{\bm{u}}}
\def\vv{{\bm{v}}}

\def\vx{{\bm{x}}}

\def\vz{{\bm{z}}}

\def\mH{{\bm{H}}}
\def\mI{{\bm{I}}}

\DeclareMathAlphabet{\mathsfit}{\encodingdefault}{\sfdefault}{m}{sl}
\SetMathAlphabet{\mathsfit}{bold}{\encodingdefault}{\sfdefault}{bx}{n}

\def\gC{{\mathcal{C}}}
\def\gD{{\mathcal{D}}}

\def\gL{{\mathcal{L}}}

\def\gX{{\mathcal{X}}}

\usepackage{amsthm}
\theoremstyle{plain}
\newtheorem{theorem}{Theorem}[section]
\newtheorem{proposition}[theorem]{Proposition}

\theoremstyle{definition}

\theoremstyle{remark}
\newtheorem{remark}[theorem]{Remark}

\newcommand{\modelname}[0]{AVUQ}
\newcommand{\rev}[1]{{\color{blue}#1}}

\newif\ifshowedits
\showeditsfalse
\newcommand{\figtwofile}{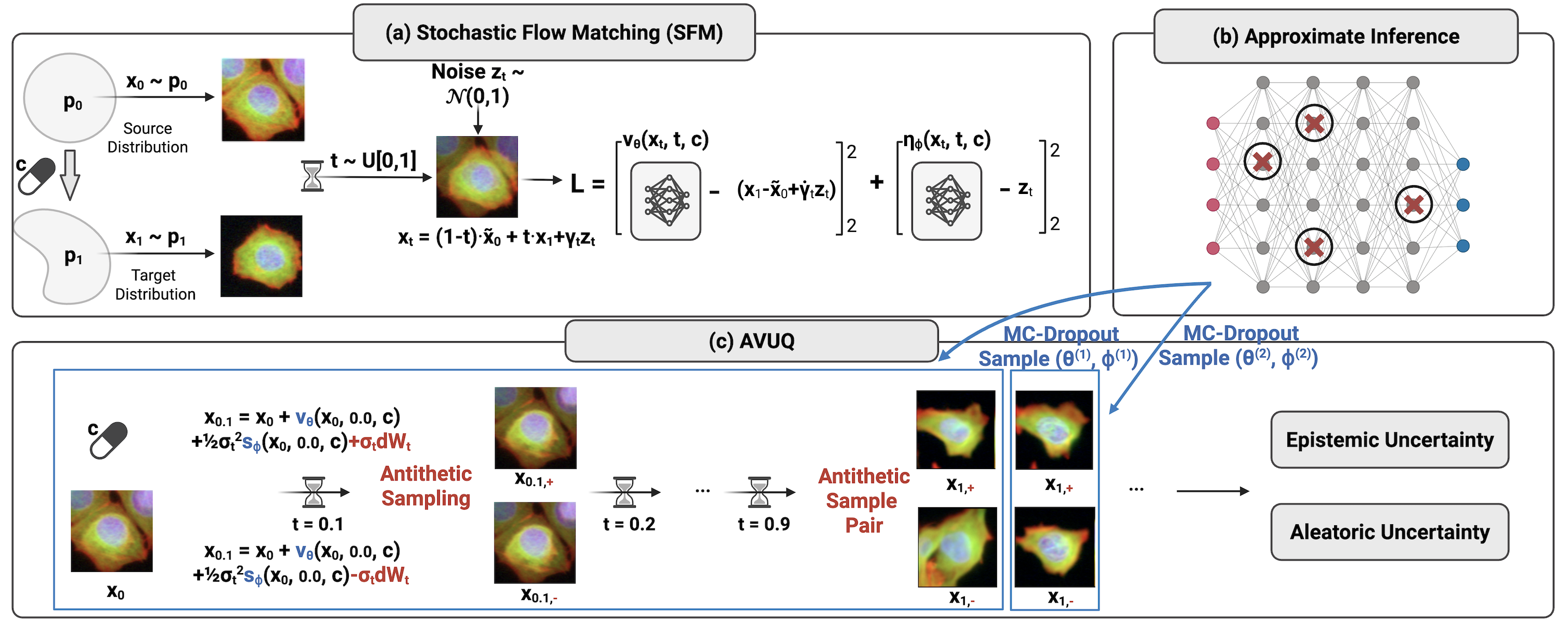}               
\ifshowedits
  \renewcommand{\figtwofile}{figures/sfm_uq_tracked.png}  
\else
  \definecolor{blue}{rgb}{0,0,0}
  \definecolor{red}{rgb}{0,0,0}
\fi

\usepackage[hidelinks]{hyperref}

\newcommand{\paperTitle}{Uncertainty-Aware Distribution-to-Distribution Flow Matching for Scientific Image Generation}
\newcommand{\thetitle}{\paperTitle} 
\hypersetup{pdftitle={\paperTitle}}

\title{\paperTitle}
\date{} 

\author{ {\hspace{0.1mm}Dongxia Wu}\\
    \texttt{Stanford University}\\
    \texttt{Stanford, CA}\\
	\texttt{dowu@stanford.edu} \\
    \And{\hspace{0.1mm}Yuhui Zhang}\\
    \texttt{Stanford University}\\
    \texttt{Stanford, CA}\\
    \texttt{yuhuiz@stanford.edu}\\
    \And{\hspace{0.1mm}Serena Yeung-Levy}\\
    \texttt{Stanford University}\\
    \texttt{Stanford, CA}\\
    \texttt{syyeung@stanford.edu} \\
    \And{\hspace{0.1mm}Emma Lundberg}\\
    \texttt{Stanford University}\\
    \texttt{Stanford, CA}\\
	\texttt{emmalu@stanford.edu} \\
    \And{\hspace{0.1mm}Emily B. Fox}\\
    \texttt{Stanford University}\\
    \texttt{Stanford, CA}\\
	\texttt{ebfox@stanford.edu} \\
}

\begin{document}
\maketitle
\begin{abstract}
  Distribution-to-distribution generative models support scientific image generation tasks
  ranging from modeling cellular perturbation responses to translating medical
  images across conditions. Trustworthy generation requires \emph{reliability},
  or generalization across labs, devices, and experimental conditions, and
  \emph{accountability}, or detecting out-of-distribution cases where
  predictions may be unreliable. We leverage Stochastic Flow Matching (SFM), a
  marginal-preserving stochastic extension of flow matching for improved
  generalization under distribution shift. SFM augments deterministic flows with
  a diffusion term together with a learned score-based drift correction,
  retaining the learned transport marginals while modeling conditional
  variability. Building on SFM, we introduce Bayesian Stochastic
  Flow Matching (BSFM) as a companion uncertainty quantification mechanism and
  develop \modelname~(Antithetic Variance-reduction Uncertainty Quantification)
  to approximately estimate epistemic and aleatoric uncertainty via
  sample-efficient antithetic sampling with approximate posterior inference. We
  further use \modelname~to yield anomaly scores for unreliable generation
  detection.
  Experiments on cellular imaging (BBBC021, JUMP) and brain fMRI (Theory of
  Mind) across diverse unseen scenarios show that SFM improves generalization
  while \modelname~provides effective uncertainty-based anomaly scores under
  practical sampling budgets.
\end{abstract}

\section{Introduction}

Distribution-to-distribution generative image models predict cellular morphology
after chemical or genetic perturbations \citep{zhang2025cellflux}, map resting-state
to task-evoked brain activity \citep{kwon2025predicting, kan2022fbnetgen}, and
translate medical images between control and treatment conditions
\citep{li2023zero,arslan2025self}. By learning transformations between source and
target distributions, they capture biological responses to interventions,
neural activity shifts, and manifestations of pathology.
Scientific deployment requires more than generative quality
\citep{huang2025trustworthiness, blau2024protecting}; it also requires \textbf{reliability}
across unseen perturbation types, cell lines, assay plates, laboratories,
imaging devices, and protocols; and \textbf{accountability}, the capacity to
detect out-of-distribution (OOD) cases and unreliable generations. Otherwise,
misleading predictions can prompt costly validation of false hypotheses,
create patient safety concerns, or miss scientific opportunities. In drug
discovery, confidently predicted responses to novel compounds could misdirect
expensive screening campaigns. Fig.~\ref{fig:motivation} (left) illustrates
these challenges.

\begin{figure*}[t]
    \centering
    \includegraphics[width=\linewidth]{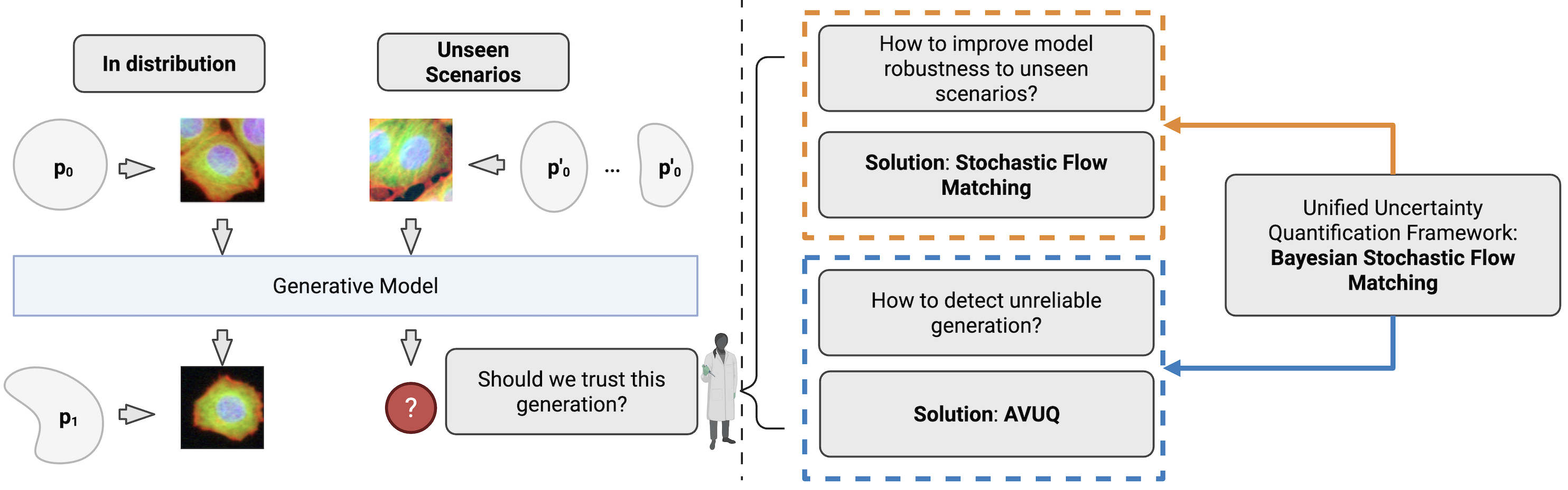}
    \vspace{-1.2em}
    \caption{Trustworthy distribution-to-distribution generative modeling. \textbf{Left:}
    In-domain samples are drawn from the source distribution $p_0$, while unseen samples
    may arise from shifted source distributions $\{p_0^\prime\}$. \textbf{Right:} Our
    approach centers on Stochastic Flow Matching (SFM) to improve model
    generalization, with Bayesian Stochastic Flow Matching (BSFM) and
    \modelname~providing approximate aleatoric and epistemic uncertainty estimates
    for unreliable generation detection.}
    \label{fig:motivation}
    \vspace{-0.7em}
\end{figure*}

Flow matching \citep{lipman2022flow,zhang2025cellflux} is a natural backbone for
distribution-to-distribution generation because it learns continuous
transformations via neural ordinary differential equations (ODEs). However,
standard flow matching is deterministic: for a fixed input and condition, it
produces a single generated output. This can limit generalization under
scientific distribution shifts that require the model to represent conditional
variability rather than a single transport path. A naive stochastic extension
that simply injects diffusion noise can corrupt the learned marginals and
degrade generation quality. Thus, a key technical challenge is to add
stochasticity in a way that improves robustness while preserving the learned
distribution-to-distribution mapping.

We show that \textbf{Stochastic Flow Matching (SFM)} \citep{song2020score,karras2022elucidating, su2025three} can be used to address this 
challenge by adding stochasticity via a
marginal-preserving stochastic extension of deterministic flow matching. SFM
introduces a diffusion term together with a learnable score-based drift
correction, so stochastic sampling can model conditional variability without
changing the target marginals of the learned transport. Although this model 
has been \rev{used for noise-to-data generation \citep{song2020score,karras2022elucidating} and to improve in-distribution quality in distribution-to-distribution generation \citep{su2025three}}, we are the first to show that SFM
improves generalization to unseen scenarios in distribution-to-distribution scientific
imaging: it retains the efficiency and regression-based training of flow matching while
improving robustness to unseen scenarios.

Uncertainty quantification (UQ) provides a complementary accountability
mechanism, but scaling UQ to high-dimensional image generators is
computationally challenging. \rev{Existing UQ methods for diffusion and flow models target noise-to-data generation and mostly estimate a single uncertainty \citep{kou2023bayesdiff,jazbec2025generative,han2026flow}; separating aleatoric and epistemic uncertainty has required costly nested sampling \citep{chan2024estimating}.} \textbf{Bayesian Stochastic Flow Matching} (BSFM)
provides a promising UQ framework, but exact Bayesian inference is intractable for the
velocity and score networks, and naive nested sampling requires many costly SDE solves
to obtain stable estimates\rev{: at affordable budgets, aleatoric Monte Carlo noise leaks into the nested epistemic estimate}. We therefore develop \textbf{\modelname}~(Antithetic Variance-reduction Uncertainty Quantification) by leveraging antithetic sampling for
reduced-variance, sample-efficient estimation of both aleatoric and epistemic uncertainty terms.
Although \modelname~can be applied with any method of generating approximate posterior samples, in
this paper we use MC-Dropout (MCD) as a scalable posterior approximation.
We show empirically that this MCD instantiation outperforms other approximate
inference methods. The key contribution, however, is the sample-efficient
variance reduction provided by \modelname; without it, nested sampling
would be computationally prohibitive. The resulting approximate estimates of aleatoric and epistemic uncertainty serve as anomaly scores for unreliable
generation detection without architectural changes or additional training.

Through extensive experiments on cellular imaging datasets (BBBC021, JUMP) and
an fMRI dataset (`Theory of Mind', ToM) across diverse unseen scenarios---novel
perturbations, different laboratory settings, novel cell lines, unseen assay plates,
and low-performing ToM subjects---we demonstrate the effectiveness of SFM and
\modelname. Specifically, SFM improves reliability through enhanced
generalization under distribution shifts, while \modelname~provides
sample-efficient uncertainty estimates that support effective unreliable
generation detection.
Together, these components provide a practical recipe for more trustworthy
distribution-to-distribution generative modeling in scientific applications
(Fig.~\ref{fig:motivation}, right). In summary, our contributions are:
\begin{itemize}
    \item We show \textbf{Stochastic Flow Matching (SFM)}, a marginal-preserving stochastic
    extension of flow matching can be used to improve generalization under distribution shifts.
    \item Proposing \textbf{\modelname}~(Antithetic Variance-reduction Uncertainty Quantification) in the \textbf{Bayesian Stochastic Flow Matching (BSFM)} framework, which provides approximate and sample-efficient estimates of aleatoric and epistemic uncertainty.
    \item Demonstration across scientific imaging benchmarks that SFM improves
    generation quality in unseen scenarios and that \modelname~provides effective
    uncertainty-based signals for unreliable generation detection in OOD scenarios.
\end{itemize}

\section{Related Work}

\paragraph{Distribution-to-distribution Generative Models.}
Distribution-to-distribution models map between two data distributions rather
than from noise to data, with applications across scientific domains
\citep{zhang2025cellflux,kwon2025predicting,kan2022fbnetgen,li2023zero,arslan2025self}.
Approaches include flow-based models \citep{lipman2022flow,lipman2024flow},
diffusion bridges \citep{shi2023diffusion,zhou2023denoising,de2021diffusion},
and generative adversarial networks
\citep{zhu2017unpaired,park2020contrastive,kim2023unpaired}. Flow matching
trains by velocity-field regression and generates through ODE integration.
Extensions use classifier-free guidance for conditional generation
\citep{zhang2025cellflux} and stochastic perturbations for improved
in-distribution performance \citep{su2025three}. Despite high in-distribution
fidelity, behavior under distribution shifts remains unexplored, motivating
UQ for unreliable generation.

\paragraph{Uncertainty Quantification in Generative Models.}
Diffusion-model UQ has drawn growing interest. Post-hoc last-layer Laplace
methods such as BayesDiff \citep{kou2023bayesdiff} and Jazbec et al.
\citep{jazbec2025generative} quantify epistemic uncertainty; denoiser ensembles
such as DECU \citep{berry2024shedding} capture model variability. Posterior
sampling quantifies aleatoric uncertainty \citep{de2025diffusion}. Applications
include inverse imaging \citep{xie2022measurement,feng2023score,wu2024principled}
and text-to-image generation \citep{franchi2025towards}; conformal methods
\citep{teneggi2023trust,ekmekci2025conformalized} provide distribution-free
coverage. Hyper-Diffusion \citep{chan2024estimating} decomposes both uncertainty
types in one hypernetwork but uses costly nested sampling. Existing methods
primarily address \textit{noise-to-data} diffusion generation; UQ for
\textit{flow-based} \textit{distribution-to-distribution} generation remains
unexplored.

\paragraph{Out-of-Distribution Detection in Generative Models.} 
Generative-model OOD detection commonly uses reconstruction or manifold
consistency, including diffusion inpainting \citep{liu2023unsupervised},
denoising trajectories \citep{graham2023denoising}, semantic mismatch
\citep{gao2023diffguard}, and projection regret \citep{choi2023projection}.
Other cues include diffusion-path curvature \citep{heng2024out}, covariance
spectra \citep{shoushtari2025eigenscore}, and norm-guided residuals
\citep{zhang2025diffusionad}. Applications span medical imaging
\citep{graham2023unsupervised,bercea2024diffusion,linmans2024diffusion}, road
scenes \citep{galesso2024diffusion}, zero-shot \citep{abdi2025zero}, and
trend-based \citep{kim2024unsupervised} anomaly detection. These methods define
OOD globally against one data distribution in noise-to-data or
image-to-manifold settings. We address distribution-to-distribution settings,
where OOD is conditional on the source distribution; this source-aware regime
has not been addressed by diffusion or flow-matching approaches.

\section{Preliminaries}
\subsection{Distribution-to-Distribution Image Generation}
Let $p_0$ be
a source distribution and $p_1$ a target distribution over the image space 
$\gX$. The goal of distribution-to-distribution image generation is to
learn a generative model mapping $p_0$ to $p_1$.
Given an image $\vx_0 \sim p_0$,
the model generates $\vx_1 \sim p_1$ representing 
the transformed image in the target distribution.
Optionally, we can incorporate a condition space $\gC$ to enable
conditional generation, where the target distribution becomes $p_1(\cdot|c)$ for
$c \in \gC$. We then learn a
conditional generative model $p(\vx_1 | \vx_0, c)$ that maps
$\vx_0 \sim p_0$ to $\vx_1 \sim p_1(\cdot|c)$. In cellular imaging, $p_0$ and $p_1$ represent the distributions of unperturbed
cell images and images under drug perturbation $c$, respectively. The model predicts morphological changes induced
by the perturbation.

\subsection{Flow Matching-based Generative Modeling}
Flow matching-based generative models learn invertible mappings between $p_0$ and
$p_1$ via continuous transformations. Given pairs of samples from
these distributions, flow matching learns a time-dependent velocity field using
a neural network $\vv_\theta: \gX \times [0,1] \to \gX$ that
defines the instantaneous direction and magnitude of change at each
point. The transformation process follows the ODE:
\begin{align}
\frac{d\vx_t}{dt} = \vv_\theta(\vx_t, t), \quad \vx_0 \sim p_0, \vx_1 \sim p_1, t \in [0,1].
\end{align}

We employ the rectified flow formulation \citep{liu2022flow}, which
yields a straight-line path:
\begin{align}
    \vx_t = (1-t)\vx_0 + t\vx_1, \quad t \sim \mathcal{U}[0,1].
\end{align}
The linear path has velocity field $\vv(\vx_t,t) = d\vx_t/dt =
\vx_1 - \vx_0$, representing the optimal transport direction at
each point. The neural network $\vv_\theta$ is trained to
match this velocity field by minimizing:
\begin{align}
    \gL(\theta) = \mathbb{E}_{(\vx_0, \vx_1) \sim p_0 \times p_1, t \sim \mathcal{U}(0,1)}\|\vv_\theta(\vx_t, t) - (\vx_1 - \vx_0)\|^2_2.
\end{align}

At inference, given $\vx_0 \sim p_0$, we solve the
ODE from $t=0$ to $t=1$ to obtain the deterministic output image $\hat{\vx}_1$:
\begin{align}
    \hat{\vx}_1 = \vx_0 + \int_0^1 \vv_\theta(\vx_t, t) dt.
\end{align}

\subsection{Conditional Flow Matching}
Conditional flow matching (CFM)~\citep{zhang2025cellflux} learns a conditional vector field
$\vv_\theta: \gX \times [0,1] \times \gC \to \gX$ defining
the flow from $p_0$ to a condition-dependent target distribution $p_1(\cdot|c)$ for $c \in \gC$ via the ODE:
\begin{align}
\frac{d\vx_t}{dt} = \vv_\theta(\vx_t, t, c), \quad \vx_0 \sim p_0, \vx_1 \sim p_1(\cdot|c).
\end{align}

We apply classifier-free guidance \citep{ho2022classifier,zheng2023guided} to enhance conditional generation.
During training, we randomly mask condition $c$ with probability $p_c$,
replacing it with a null condition $c_\emptyset$. At inference, conditional
and unconditional predictions combine to form the guided vector field:
\begin{align}
\tilde{\vv}_\theta(\vx_t, t, c) = \alpha\vv_\theta(\vx_t, t, c) + (1-\alpha)\vv_\theta(\vx_t, t, c_\emptyset),
\end{align}
where $\alpha \geq 1$ controls the guidance strength.

Henceforth, we use the conditional notation $p(\vx_1 | \vx_0, c)$ and $p_1(\cdot|c)$
as a unified representation for both conditional and unconditional cases.
The unconditional case can be recovered via $c = c_\emptyset$.

\section{Methodology}
We first deploy \textbf{Stochastic Flow Matching (SFM)},
which stochastically maps inputs $\vx_0$ to samples $\vx_1$ conditioned
on $c$ while preserving the learned transport marginals. We then
formulate \textbf{Bayesian Stochastic Flow Matching (BSFM)} and propose \textbf{\modelname}~for
sample-efficient nested estimation of aleatoric and
epistemic uncertainty. Fig.~\ref{fig:method_overview} outlines our methodological components.
\begin{figure*}[t]
    \centering
    \includegraphics[width=0.95\textwidth]{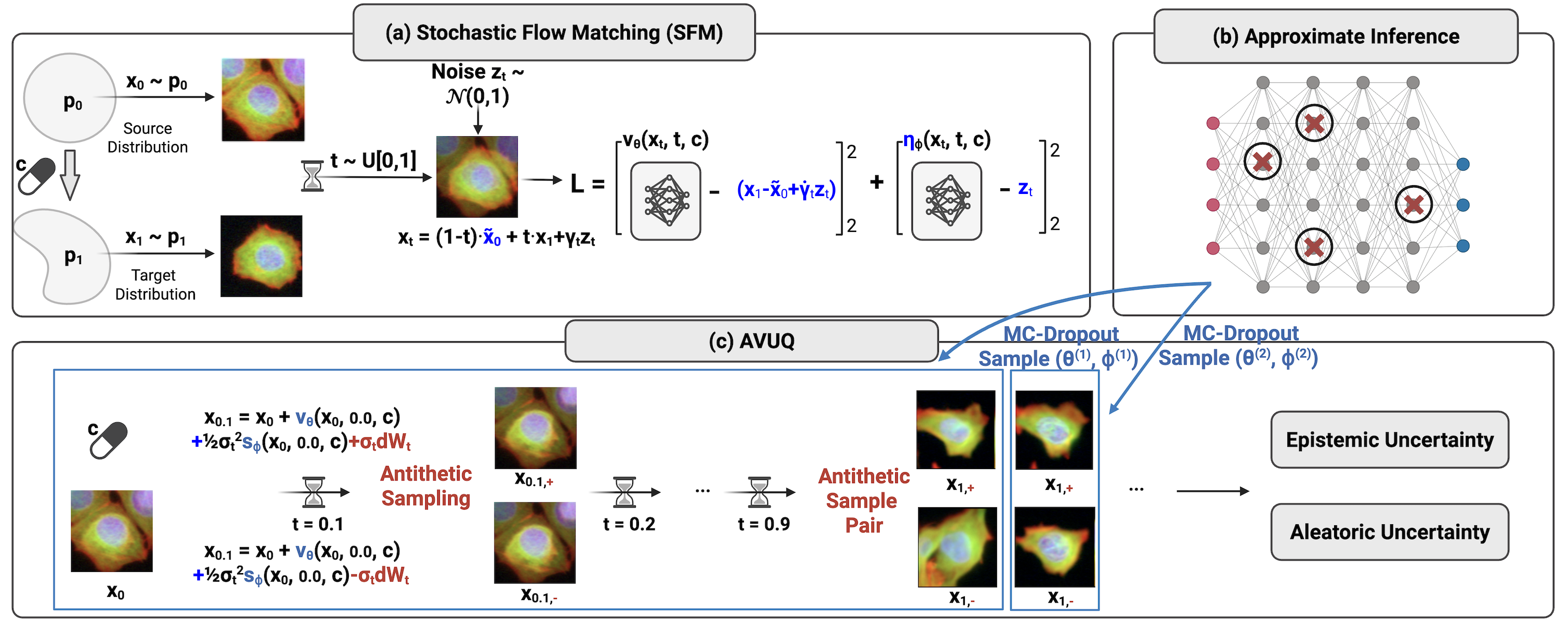}
    \vspace{-0.8em}
    \caption{Core components of the Bayesian Stochastic Flow Matching (BSFM) framework:
    (a) Stochastic Flow Matching (SFM) training process.
    (b) MC-Dropout as one approximate inference method that can be used by \modelname.
    (c) \modelname~with antithetic sampling.
    }
    \label{fig:method_overview}
    \vspace{-0.7em}
\end{figure*}
\subsection{Stochastic Flow Matching}

To quantify the aleatoric uncertainty inherent to the transformation process,
a naive extension adds diffusion noise to the ODE solver of the learned velocity field:
\begin{align}
d\vx_t = \vv_\theta(\vx_t, t, c) dt + \sigma_t d\vW_t,
\end{align}
where $\sigma_t$ is a predefined noise schedule and $\vW_t$ denotes standard Brownian motion.
However, this directly perturbs the marginals $p_t(\vx_t|c)$ at each time $t$,
deviating from the learned transport and degrading generation quality.
Following prior work~\citep{su2025three}, we instead derive a marginal-preserving stochastic differential equation (SDE) via the Fokker-Planck equation.
For any ODE $d\vx_t = \vv_\theta(\vx_t,t,c)\,dt$ with marginals $\{p_t(\vx_t|c)\}$,
the corresponding SDE with drift correction,
\begin{align}
\color{blue}d\vx_t = \Big(\vv_\theta(\vx_t, t, c) + \tfrac{1}{2}\sigma_t^2 \nabla_{\vx_t}\log p_t(\vx_t|c)\Big)dt + \sigma_t d\vW_t
\label{eq:sde_correct}
\end{align}
shares identical marginals $p_t(\vx_t|c)$ for all $t\in[0,1]$ \rev{when integrated forward from $t=0$ to $t=1$ (derivation in Appendix~\ref{app:sde_derivation})}.
The score correction term compensates for diffusion-induced drift,
preserving the learned distribution-to-distribution mapping while introducing controlled stochasticity.

\paragraph{Training and Inference.}
We parameterize the score function as $s_\phi(\vx_t, t, c) \approx \nabla_{\vx_t}\log p_t(\vx_t|c)$
and train it jointly with the velocity field. We perturb the interpolation path
following the standard score matching construction~\citep{albergo2023stochastic}:
\begin{align}
\color{blue}\vx_t = (1-t)\tilde{\vx}_0 + t\vx_1 + \gamma_t \vz_t,\quad \vz_t\sim\mathcal{N}(\vzero, \mI),
\end{align}
where $\gamma_t$ is a smooth noise schedule with $\gamma_0=\gamma_1=0$\rev{, and $\tilde{\vx}_0=\vx_0+\boldsymbol{\epsilon}$, $\boldsymbol{\epsilon}\sim\mathcal{N}(\vzero,\mI)$, is a source perturbation applied with probability $0.5$ during training \citep{su2025three}}.
\rev{For $\gamma_t>0$, $\nabla_{\vx_t}\log p_t(\vx_t|c)=-\gamma_t^{-1}\mathbb{E}[\vz_t\mid\vx_t,c]$, so we learn a noise predictor $\boldsymbol{\eta}_\phi(\vx_t,t,c)\approx\mathbb{E}[\vz_t\mid\vx_t,c]$ and set $s_\phi=-\boldsymbol{\eta}_\phi/\gamma_t$.}
We set $\gamma_t=a \sin^2(\pi t)$ because it vanishes at the endpoints,
preserving the source and target samples, while smoothly placing the largest
perturbation near the middle of the interpolation path. We minimize the
combined objective $\gL(\theta,\phi)=\gL_{\vv}(\theta)+\lambda\gL_{\rev{\eta}}(\phi)$,
\begin{align}
\color{blue}\gL_{\vv}(\theta)&\color{blue}=\mathbb{E}\|\vv_\theta(\vx_t,t,c)-(\vx_1-\tilde{\vx}_0+\dot{\gamma}_t\vz_t)\|_2^2, \nonumber\\
\color{blue}\gL_{\eta}(\phi)&\color{blue}=\mathbb{E}\|\boldsymbol{\eta}_\phi(\vx_t,t,c)-\vz_t\|_2^2,
\end{align}
where expectations are taken over $(\vx_0,\vx_1)\sim p_0\times p_1(\cdot|c)$ and $t\sim\mathcal{U}(0,1)$.
At inference, we solve the SDE using the learned velocity and score networks to
generate samples, with $\vv_\theta$ denoting the conditional or guided velocity
field from the preliminaries:
\begin{align}
\color{blue}d\vx_t = \Big(\vv_\theta(\vx_t, t, c) + \tfrac{1}{2}\sigma_t^2 s_\phi(\vx_t, t, c)\Big)dt + \sigma_t d\vW_t.
\label{eq:sde_learned}
\end{align}

\subsection{Bayesian Stochastic Flow Matching}

To quantify epistemic uncertainty,
we adopt a Bayesian treatment of the velocity and score network parameters $(\theta, \phi)$.
Given training data $\gD$, the posterior is given by
\begin{align}
p(\theta, \phi | \gD) \propto p(\gD | \theta, \phi) p(\theta, \phi),
\label{eq:bayes_rule}
\end{align}
where $p(\gD | \theta, \phi)$ denotes the likelihood and $p(\theta, \phi)$ the prior. The predictive distribution marginalizes over parameter uncertainty:
\begin{align}
p(\vx_1 | \vx_0, c, \gD) = \iint p(\vx_1 | \vx_0, c, \theta, \phi) p(\theta, \phi | \gD) d\theta d\phi,
\end{align}
where $p(\vx_1 | \vx_0, c, \theta, \phi)$ is the conditional distribution induced by the SDE
in \eqref{eq:sde_learned}.

\subsection{Uncertainty Decomposition and Estimation}

Total predictive uncertainty
arises from two sources:
aleatoric uncertainty, reflecting inherent stochasticity in the transformation captured by $p(\vx_1 | \vx_0, c, \theta, \phi)$,
and epistemic uncertainty, arising from limited training data and captured by the model posterior $p(\theta, \phi | \gD)$.

We quantify uncertainty via total predictive variance and decompose it 
using the law of total variance:
\begin{align}
\text{Var}(\vx_1 | \vx_0, c, \gD) =
&\underbrace{\mathbb{E}_{p(\theta, \phi | \gD)}
[\text{Var}(\vx_1 | \vx_0, c, \theta, \phi)]}_{\text{aleatoric uncertainty}} \nonumber\\
&+ \underbrace{\text{Var}_{p(\theta, \phi | \gD)}
[\mathbb{E}(\vx_1 | \vx_0, c, \theta, \phi)]}_{\text{epistemic uncertainty}},
\label{eq:total_variance}
\end{align}
where the first term represents aleatoric uncertainty and the second represents
epistemic uncertainty. \rev{Here, the aleatoric term measures how much generations from a single model vary for a given input, and the epistemic term measures how much the model's average generation moves across plausible parameters under the (approximate) posterior. With unpaired source and target data, SFM fixes the marginals but not the conditional spread, which is set by the stochastic coupling induced by $\gamma_t$ and $\sigma_t$; the aleatoric term is therefore model-induced conditional variance rather than biological variability, and we use both terms to rank inputs rather than as absolute measures.}

We employ nested sampling to estimate both uncertainty terms: for each of $M$ posterior samples
$(\theta^{(m)}, \phi^{(m)}) \sim p(\theta, \phi|\gD)$, we solve the SDE $K$ times
to generate $\vx_1^{(k,m)} \sim p(\cdot | \vx_0, c, \theta^{(m)}, \phi^{(m)})$ and estimate
\begin{gather}
\hat{\mu}^{(m)} = \frac{1}{K}\sum_{k=1}^K \vx_1^{(k,m)} \approx \mathbb{E}(\vx_1 | \vx_0, c, \theta^{(m)}, \phi^{(m)}),\quad\text{and}\displaybreak[1]\\
\begin{split}
\hat{\Sigma}^{(m)}
&= \frac{1}{K-1}\sum_{k=1}^K
(\vx_1^{(k,m)} - \hat{\mu}^{(m)})(\vx_1^{(k,m)} - \hat{\mu}^{(m)})^T\\
&\approx \text{Var}(\vx_1 | \vx_0, c, \theta^{(m)}, \phi^{(m)}).
\end{split}
\end{gather}

For our estimate of aleatoric uncertainty, we compute
\begin{equation}
    \hat{A} = \frac{1}{M}\sum_{m=1}^{M}\hat{\Sigma}^{(m)}
\end{equation}

For epistemic uncertainty, we note that by the law of total variance (conditioning on $\vx_0, c$),
\begin{equation}
\begin{split}
\mathrm{Var}_{p(\theta,\phi \mid \gD)}[\hat{\mu}(\theta,\phi)]
&= \underbrace{\mathrm{Var}_{p(\theta,\phi \mid \gD)}\!\left[\mathbb{E}\!\big(\hat{\mu}\mid \theta,\phi\big) \right]}_{\text{epistemic uncertainty}}\\
&\quad+
\underbrace{\mathbb{E}_{p(\theta,\phi \mid \gD)}\!\left[\mathrm{Var}\!\big(\hat{\mu}\mid\theta,\phi\big)\right]}_{\text{residual Monte Carlo noise}} .
\end{split}
\label{eq:MCnoise}
\end{equation}
The first term equals the true epistemic uncertainty because
$\mathbb{E}\!\big(\hat{\mu}\mid \theta,\phi\big) = \mathbb{E}(\vx_1 | \vx_0, c, \theta, \phi)$.
Additionally, note that
$\mathrm{Var}\!\big(\hat{\mu}\mid\theta,\phi\big) =
\frac{1}{K}\mathrm{Var}(\vx_1 | \vx_0, c, \theta, \phi)$. That is, the
residual Monte Carlo (MC) noise is aleatoric uncertainty contamination, scaled
by the sample size $K$, that appears in the naive estimator for epistemic uncertainty:
\begin{equation}
    \hat{E} = \frac{1}{M-1}\sum_{m=1}^M (\hat{\mu}^{(m)} - \bar{\mu})(\hat{\mu}^{(m)} - \bar{\mu})^T.
\end{equation}
Ideally, we would correct for this MC noise by subtracting $\hat{A}/K$. However, nested sampling incurs $\mathcal{O}(M \times K)$ SDE solves,
which is computationally prohibitive for high-dimensional images, especially for large $M$ and $K$, the regime in which our estimates are lower variance.
To handle this, we introduce a sample-efficient alternative via \modelname. \rev{At our budgets, subtracting $\hat{A}/K$ leaves per-input estimates dominated by sampling noise (Appendix~\ref{app:avuq_variance_reduction}), so we use the \modelname~traces as ranking statistics rather than as an accurate per-input decomposition.}

\subsection{Scalable and Sample-Efficient Uncertainty Estimation}
\label{sec:scalable}
This section addresses two computational bottlenecks: sampling from an intractable neural-network posterior and estimating SDE moments with few solver calls.
\paragraph{\modelname}
Generating nested SDE samples $\vx_1^{(i,m)} \sim p(\cdot | \vx_0, c, \theta^{(m)}, \phi^{(m)})$
for $i=1,\ldots,K$ and $m=1,\ldots,M$ is computationally expensive.
To alleviate this cost, we employ antithetic sampling \citep{hammersley1956new} to reduce
MC variance under limited sampling budgets. Specifically, for each noise
realization $\epsilon$ used to drive the SDE, we additionally simulate the trajectory induced by
its negation $-\epsilon$. This construction introduces negative correlation between paired samples, reducing conditional-mean estimator variance without bias under the usual symmetry of the driving noise.

In particular, for $J$ antithetic pairs ($K=2J$ total samples), $\big(\vx_{1,+}^{(j,m)},\vx_{1,-}^{(j,m)}\big)$, we form
\begin{equation}
\mathbf{y}^{(j,m)} = \frac{1}{2}\big(\vx_{1,+}^{(j,m)} + \vx_{1,-}^{(j,m)}\big), \quad
\hat{\mu}^{(m)}_{anti} = \frac{1}{J} \sum_{j=1}^J \mathbf{y}^{(j,m)}.
\end{equation}
We use $\hat{\mu}_{anti}^{(m)}$ as the conditional-mean estimate
in the epistemic estimator. For aleatoric uncertainty, we compute the
within-posterior-sample marginal variance from the raw antithetic trajectories
$\{\vx_{1,+}^{(j,m)},\vx_{1,-}^{(j,m)}\}_{j=1}^J$, not from the pair means. \rev{Because paired trajectories have correlation $\rho<0$, the expected sample variance is $\sigma^2(1-\rho/(K-1))$ for per-sample variance $\sigma^2$; in our experiments of Sec.~\ref{sec:setup} with $K=4$ and the correlations in Table~\ref{tab:avuq_variance_reduction} (Appendix~\ref{app:antithetic_variance}), this corresponds to a \mbox{17--26\%} inflation in sample variance.} Importantly, antithetic sampling reduces the residual MC error
in \eqref{eq:MCnoise}, enabling more reliable epistemic uncertainty estimation under
tight computational budgets. See Fig.~\ref{fig:method_overview}.

\paragraph{Estimating the Model Posterior $p(\theta, \phi \mid \gD)$.}
\modelname~relies on posterior samples $\theta^{(m)}, \phi^{(m)}$. Exact Bayesian inference for deep neural networks is typically intractable.
We therefore adopt MC-Dropout, a scalable approach
to approximate the posterior \citep{gal2016dropout}. As shown in
Fig.~\ref{fig:method_overview}, we insert dropout layers in the velocity
and score networks. At inference, we keep dropout active and perform multiple
stochastic forward passes, yielding samples $\{(\theta^{(m)},\phi^{(m)})\}_{m=1}^M$
from the approximate posterior. This provides a simple and scalable method
in high-dimensional parameter spaces. We note that \modelname~works with
any approximate posterior sampling method. See Sec.~\ref{sec:results} for empirical comparisons.
Here we focus on using \modelname~for unreliable generation detection.
Assessing conditional calibration is challenging without matched source--target
observations in this high-dimensional, unpaired setting; we leave this evaluation
to future work.

\paragraph{\modelname~for unreliable generation detection.}

For our scalar anomaly scores, $f(\vx_0,c)$, we compute
$-\mathrm{tr}(\hat{E})$ and $-\mathrm{tr}(\hat{A})$ for epistemic and
aleatoric uncertainty, respectively. These scores are traces, so we never form
or store full pixel-level covariance matrices. Since paired trajectories are correlated, this aleatoric trace is used as a
practical scalar ranking score rather than an unbiased independent-trajectory
variance estimator. For MAP, we use the negative
mean pixel-wise variance as the anomaly score. We then apply a binary decision
rule:
\begin{equation}
h(\vx_0, c) = \begin{cases}
\text{Unreliable}, &  \text{if } f(\vx_0, c) > \tau \\
\text{Reliable}, & \text{if } f(\vx_0, c) \leq \tau
\end{cases}
\end{equation}
where $\tau$ is a threshold parameter. The sign convention reflects the
failure mode we observe in these generative models: under distribution shift,
the model can collapse toward overconfident predictions, so uncertainty may
decrease rather than increase on OOD inputs
\citep{kirichenko2020normalizing, zhang2021understanding}. In this regime,
unusually low epistemic uncertainty is itself an anomaly signal, motivating
$-\mathrm{tr}(\hat{E})$ rather than $\mathrm{tr}(\hat{E})$. We apply the same
convention to aleatoric uncertainty because it also decreases under the shifts
considered in Sec.~\ref{sec:results} and is particularly sensitive to mild deviations. This
direction should be revalidated in settings where OOD inputs instead induce
genuine posterior disagreement rather than collapse. \rev{A two-sided, label-free variant that scores the absolute $z$-score of each log-trace relative to ID inputs needs no direction; as expected, it has lower power, but \modelname~remains the best detector in all four cohorts (Appendix~\ref{app:two_sided}).}


\section{Results}
We assess SFM for generalization and BSFM with \modelname~for unreliable
generation detection on cellular imaging and fMRI tasks, comparing against a
range of baseline alternatives.

\subsection{Datasets}
We evaluate on two cell-imaging benchmarks: \textbf{BBBC021} (chemical perturbations)
\citep{caie2010high} and \textbf{JUMP} (chemical and genetic perturbations)
\citep{chandrasekaran2023jump}. To probe robustness and unreliable generation
detection, we define four scenarios spanning mild to severe distribution shifts:
\textbf{Unseen Plates} and \textbf{Unseen Cell Lines} in JUMP, synthetic
\textbf{Intensity Shift} in BBBC021 to mimic imaging-condition variation across
laboratory settings, and a combined \textbf{Unseen Pert.} setting with unseen
perturbations and intensity shifts. We further evaluate \textbf{Intensity Shift}
scenarios in JUMP for generation quality and unreliable generation detection, spanning a full range of
distribution shifts from mild to severe.

To further assess generalization performance, we additionally consider the \textbf{fMRI} Theory-of-Mind (ToM) dataset \citep{richardson2018development},
where the objective is to transform resting-state scans to task-activated
states. We train on high-performing ToM subjects and evaluate generalization to
low-performing (LToM) subjects.
Table~\ref{tab:dataset_summary} summarizes datasets and scenarios. More details are provided in Appendix \ref{app:dataset}.

\subsection{Experimental Setup}
\label{sec:setup}
\paragraph{Baselines.}
For generalization, we compare \textbf{SFM} against: \textbf{CellFlux}~\citep{zhang2025cellflux}, a
state-of-the-art deterministic flow matching model for distribution-to-distribution cell image
generation; \textbf{BBDM}, \textbf{GOUB}, and \textbf{UniDB}, three diffusion-bridge baselines~\citep{li2023bbdm, yue2023image, zhu2025unidb};
\textbf{UNSB}~\citep{kim2023unpaired}, which employs a multi-step GAN to learn a
Schrödinger bridge between source and target distributions; and
\textbf{SDEdit}~\citep{meng2021sdedit}, which performs partial noising of the source sample
followed by denoising to generate the target sample.
For unreliable generation detection, we establish a benchmark by comparing different UQ approaches
within the distribution-to-distribution flow matching
framework. We compare \textbf{\modelname} to i.i.d. nested sampling of BSFM with \textbf{MC-Dropout
(MCD)}~\citep{gal2016dropout}, \textbf{Laplace Approximation
(LA)}~\citep{kou2023bayesdiff}, and \textbf{Stochastic Weight
Averaging-Gaussian (SWAG)}~\citep{maddox2019simple}. 
We also include two baselines that do not
require nested sampling to provide lower-cost points in the cost-performance
tradeoff: a \textbf{MAP} baseline that approximates aleatoric uncertainty using
only the MAP estimate (see Appendix~\ref{app:proof_aleatoric_map}) and requires
only $\mathcal{O}(K)$ SDE solves, and a \textbf{MCD-DFM} baseline that applies
MC-Dropout to a Bayesian deterministic flow matching model without the SFM
stochastic extension and requires only $\mathcal{O}(M)$ posterior samples to
approximate epistemic uncertainty.

\paragraph{Training and evaluation details.}
We evaluate generation quality using Fréchet Inception Distance (FID) and
Kernel Inception Distance (KID) to measure distribution similarity.
For unreliable generation detection, we first take all scenarios unseen during
training and filter them to account for the fact that SFM can generalize well to
some unseen scenarios, as demonstrated in Table~\ref{tab:ood_generalization}.
We filter samples using prediction error measured in the feature space of a
mode-of-action pretrained classifier (BBBC021) \citep{zhang2025cellflux} or in
Structural Similarity Index Measure (SSIM) space (JUMP), and flag only
high-error cases as unreliable generations. Further details are in
Appendix~\ref{app:implementation}. Appendix~\ref{app:ood_sensitivity} reports
sensitivity analyses showing that the main unreliable generation detection
trends are robust to smaller sampling budgets and alternative scenario-filtering threshold choices. We also include OOD provenance detection, where
we do not filter unseen scenarios at all.

\begin{table*}[t]
\centering
\small
\caption{Generalization performance. We report FID and KID scores across six unseen scenarios on BBBC021, JUMP, and fMRI datasets.
Best results are in \textbf{bold}, second best are \underline{underlined}.}
\resizebox{0.85\textwidth}{!}{
\begin{tabular}{lcccccccccccc}
\toprule
& \multicolumn{4}{c}{\textbf{BBBC021}} & \multicolumn{6}{c}{\textbf{JUMP}} & \multicolumn{2}{c}{\textbf{fMRI}}\\
\cmidrule(lr){2-5} \cmidrule(lr){6-11} \cmidrule(lr){12-13}
\textbf{Method} & \multicolumn{2}{c}{Unseen Pert.} & \multicolumn{2}{c}{Intensity Shift} & \multicolumn{2}{c}{Unseen Cell Lines} & \multicolumn{2}{c}{Unseen Plates} & \multicolumn{2}{c}{Intensity Shift} & \multicolumn{2}{c}{Low ToM}\\
\cmidrule(lr){2-3} \cmidrule(lr){4-5} \cmidrule(lr){6-7} \cmidrule(lr){8-9} \cmidrule(lr){10-11} \cmidrule(lr){12-13}
& FID$\downarrow$ & KID$\downarrow$ & FID$\downarrow$ & KID$\downarrow$ & FID$\downarrow$ & KID$\downarrow$ & FID$\downarrow$ & KID$\downarrow$ & FID$\downarrow$ & KID$\downarrow$ & FID$\downarrow$ & KID$\downarrow$\\
\midrule
BBDM & 107.69 & 11.88 & 86.34 & 9.25 & 126.41 & 13.25 & 107.24 & 10.68 & 133.81 & 16.07 & 45.31 & 3.55 \\
GOUB & 51.89 & 4.71 & 36.94 & 3.46 & 62.47 & 6.62 & 57.66 & 6.14 & 126.51 & 14.49 & 33.49 & 2.63 \\
UniDB & 50.16 & 4.33 & 36.78 & 3.06 & 65.1 & 6.9 & 60.33 & 6.34 & 121.33 & 13.91 & \underline{30.23} & \underline{1.85} \\
UNSB & 88.31 & 6.27 & 56.95 & 4.7 & 57.24 & 4.19 & 45.5 & 4.04 & \underline{115.42} & \underline{12.74} & 75.54 & 8.89 \\
SDEdit & \underline{37.18} & \underline{3.48} & \underline{29.57} & \underline{2.72} & \underline{33.96} & \underline{3.11} & \textbf{13.21} & \textbf{0.9} & 122.23 & 13.32 & 56.87 & 6.15 \\
CellFlux & 103.73 & 12.76 & 62.01 & 5.84 & 39.06 & 3.4 & \underline{17.81} & \underline{1.04} & 209.66 & 25.06 & 34.86 & 2.63\\
\midrule
SFM & \textbf{33.29} & \textbf{2.02} & \textbf{28.14} & \textbf{1.87} & \textbf{25.1} & \textbf{1.75} & 18.02 & 1.23 & \textbf{107.60} & \textbf{10.53} & \textbf{25.55} & \textbf{1.47}\\
\bottomrule
\end{tabular}
}
\label{tab:ood_generalization}
\vspace{-0.5em}
\end{table*}

\begin{table*}[t]
\centering
\small
\caption{Unreliable generation detection performance. We report AUROC and AUPR scores
as mean $\pm$ standard error over 3 random seeds across 4 unreliable-generation scenarios on BBBC021 and JUMP datasets. We group methods into non-nested baselines,
nested i.i.d. baselines, and nested-antithetic \modelname.}
\resizebox{0.98\textwidth}{!}{
\begin{tabular}{lcccccccc}
\toprule
& \multicolumn{4}{c}{\textbf{BBBC021}} & \multicolumn{4}{c}{\textbf{JUMP}} \\
\cmidrule(lr){2-5} \cmidrule(lr){6-9}
\textbf{Method} & \multicolumn{2}{c}{Unseen Pert.} & \multicolumn{2}{c}{Intensity Shift} & \multicolumn{2}{c}{Unseen Cell Lines} & \multicolumn{2}{c}{Intensity Shift}\\
\cmidrule(lr){2-3} \cmidrule(lr){4-5} \cmidrule(lr){6-7} \cmidrule(lr){8-9}
& AUROC$\uparrow$ & AUPR$\uparrow$ & AUROC$\uparrow$ & AUPR$\uparrow$ & AUROC$\uparrow$ & AUPR$\uparrow$ & AUROC$\uparrow$ & AUPR$\uparrow$\\
\midrule
\textbf{non-nested} & & & & & & & & \\
MAP Aleatoric & 0.691$\pm$0.033 & 0.482$\pm$0.063 & 0.711$\pm$0.018 & 0.637$\pm$0.023 & 0.799$\pm$0.021 & 0.746$\pm$0.024  & 0.276$\pm$0.07 & 0.325$\pm$0.017 \\
MCD-DFM Epistemic & 0.621$\pm$0.056 & 0.332$\pm$0.047 & 0.654$\pm$0.051 & 0.532$\pm$0.05 & 0.645$\pm$0.03 & 0.545$\pm$0.02  & 0.425$\pm$0.047 & 0.368$\pm$0.017 \\
\midrule
\textbf{nested-i.i.d.} & & & & & & & & \\
SWAG Aleatoric & 0.662$\pm$0.044 & 0.472$\pm$0.076 & 0.627$\pm$0.059 & 0.610$\pm$0.052 & 0.855$\pm$0.005 & 0.816$\pm$0.004  & 0.357$\pm$0.080 & 0.35$\pm$0.024 \\
SWAG Epistemic & 0.676$\pm$0.063 & 0.499$\pm$0.099 & 0.645$\pm$0.064 & 0.639$\pm$0.06 & 0.832$\pm$0.001 & 0.800$\pm$0.005  & \underline{0.573}$\pm$0.071 & \underline{0.444}$\pm$0.038 \\
LA Aleatoric & 0.666$\pm$0.033 & 0.474$\pm$0.051 & 0.636$\pm$0.048 & 0.623$\pm$0.041 & \underline{0.865}$\pm$0.003 & \underline{0.832}$\pm$0.001  & 0.324$\pm$0.056 & 0.338$\pm$0.016 \\
LA Epistemic & 0.644$\pm$0.023 & 0.422$\pm$0.043 & 0.657$\pm$0.058 & 0.648$\pm$0.04 & 0.829$\pm$0.004 & 0.787$\pm$0.002  & 0.503$\pm$0.043 & 0.411$\pm$0.023 \\
MCD Aleatoric & 0.741$\pm$0.004 & 0.501$\pm$0.011 & 0.72$\pm$0.009 & 0.645$\pm$0.015 & 0.856$\pm$0.01 & 0.822$\pm$0.012  & 0.438$\pm$0.005 & 0.376$\pm$0.002 \\
MCD Epistemic & 0.731$\pm$0.012 & 0.486$\pm$0.014 & \underline{0.729}$\pm$0.01 & 0.642$\pm$0.012 & 0.833$\pm$0.022 & 0.796$\pm$0.029  & 0.418$\pm$0.009 & 0.368$\pm$0.004 \\
\midrule
\textbf{nested-antithetic} & & & & & & & & \\
\modelname~Aleatoric & \underline{0.742}$\pm$0.018 & \underline{0.523}$\pm$0.06 & \underline{0.729}$\pm$0.03 & \underline{0.663}$\pm$0.04 & \bf{0.869}$\pm$0.007 & \bf{0.835}$\pm$0.006 & 0.309$\pm$0.072 & 0.335$\pm$0.02 \\
\modelname~Epistemic & \bf{0.809}$\pm$0.019 & \bf{0.698}$\pm$0.059 & \bf{0.785}$\pm$0.026 & \bf{0.768}$\pm$0.033 & 0.794$\pm$0.006 & 0.726$\pm$0.01 &  \bf{0.89}$\pm$0.005 & \bf{0.799}$\pm$0.014 \\
\bottomrule
\end{tabular}
}
\label{tab:unreliable_generation_detection}
\end{table*}

Following CellFlux~\citep{zhang2025cellflux}, we adopt a U-Net-based architecture to parameterize
the velocity field. Perturbation conditions are encoded using
IMPA~\citep{palma2025predicting}. We train models for 100 epochs on BBBC021,
and 200 epochs on JUMP and fMRI using two NVIDIA H100 GPUs.
For generation quality, we evaluate using 5000 generated images per unseen scenario
on BBBC021 and JUMP, and 1024 images on fMRI.
For unreliable generation detection, we evaluate on randomly selected 500 ID and 500 OOD images per scenario
(except 250 OOD images for Unseen Pert.\ due to data availability).
For MAP and MCD-DFM, we generate $4$ samples per image.
For \modelname, MCD, LA, and SWAG, we 
generate $4\times 4$ samples per image via nested sampling, where the $4$ SDE samples for \modelname~come from $2$ antithetic pairs.

\subsection{Main Results}
\label{sec:results}
\paragraph{Stochastic Flow Matching Improves Generalization.} 
As shown in Table~\ref{tab:ood_generalization}, SFM demonstrates significant and consistent improvements in generalization
across all scenarios except Unseen Plates, where the distribution shift is the smallest and all baselines perform well. 
On JUMP Intensity Shift, SFM achieves FID/KID of 107.60/10.53, compared with 115.42/12.74 for the next-best UNSB.
This overall improved robustness stems from explicitly modeling aleatoric uncertainty through
stochastic perturbations during training and SDE-based sampling at inference.
By capturing the inherent variability in the conditional distribution
$p(\vx_1 \mid \vx_0, c, \theta, \phi)$\footnote{In principle, integrating over posterior uncertainty in $(\theta,\phi)$ could further improve generalization; however, averaging generated images $\vx_1$ pixel-wise reduces image quality.}
rather than learning a deterministic mapping, the model avoids overfitting to
spurious correlations specific to the training distribution. The resulting
uncertainty-aware generation process produces diverse, plausible outputs that
better generalize to unseen distribution shifts. \rev{SFM does not, however, make generation trustworthy by itself: under the combined shift, the MoA accuracy of its generations remains far below that of real images (Appendix~\ref{app:moa_classification}), which motivates unreliable generation detection.}
Controlled training and inference ablations appear in Appendix~\ref{app:sfm_component_ablation}.
Examples of the generated images together with
the corresponding source and ground truth target images are shown in 
Fig.~\ref{fig:ood_examples}. \rev{In the BBBC021 examples, the targets show enlarged, spread cells with a large nucleus (blue) and peripheral actin (red); SFM generates this compound-induced morphology under both shifts, whereas SDEdit and UNSB largely keep the crowded layout of the source and CellFlux shows color artifacts.}
\begin{figure}[t]
    \centering
    \includegraphics[width=0.6\linewidth]{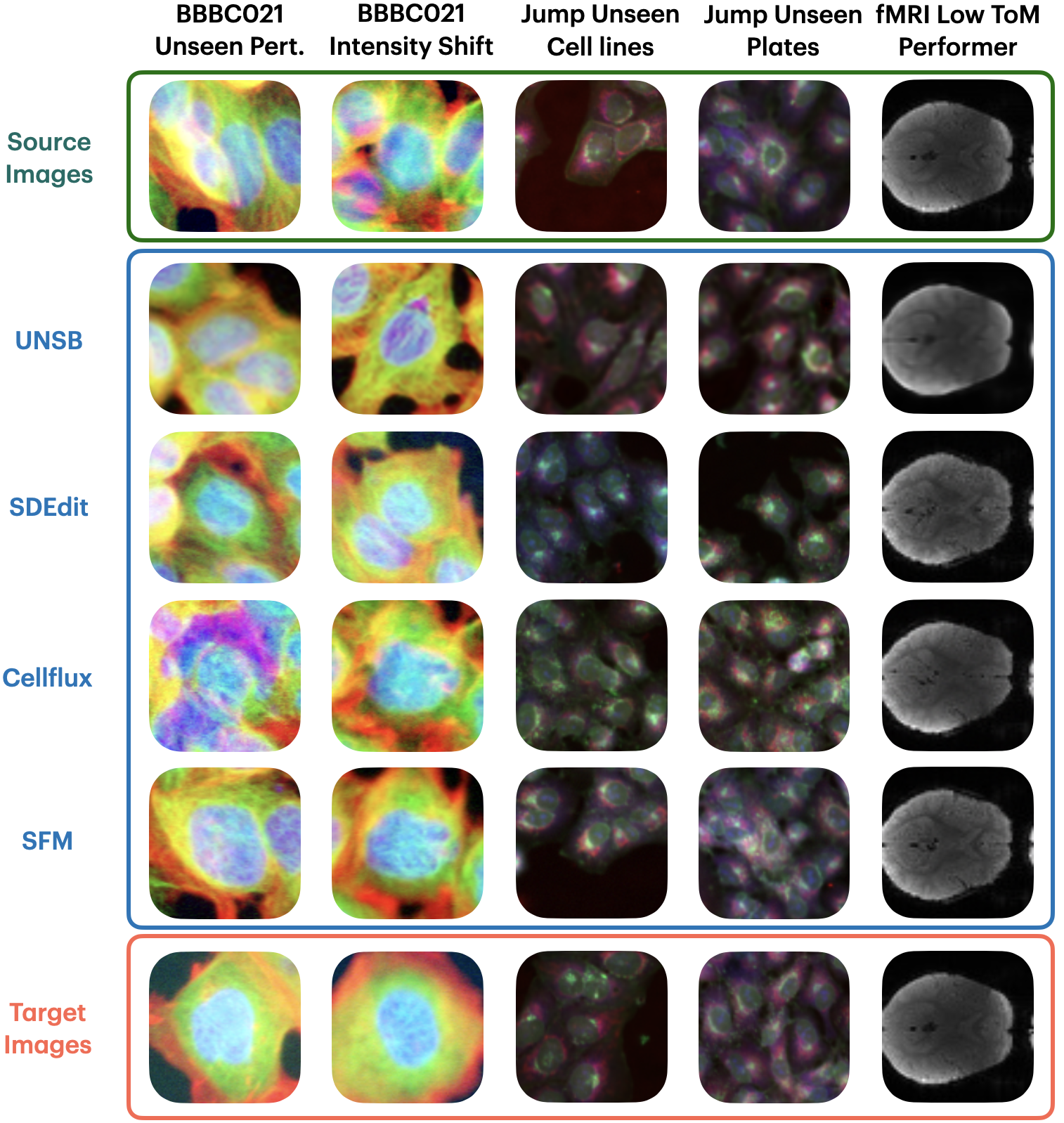}
    \caption{Example generated images from different methods under various unseen scenarios, compared with source and ground-truth target images.}
    \label{fig:ood_examples}
\end{figure}

\paragraph{Unreliable Generation Detection Performance Analysis.}
As shown in Table~\ref{tab:unreliable_generation_detection}, uncertainty-based
scores provide informative signals for unreliable generation detection across
cellular imaging shifts. Among non-nested baselines, MAP provides a strong
lower-cost aleatoric baseline in several cases; however, it significantly
underperforms \modelname~in all scenarios. MCD-DFM is generally weaker than
nested-i.i.d. BSFM baselines, indicating that the SFM stochastic backbone is
important for UQ. \modelname~provides the most consistent uncertainty signals
for unreliable generation detection, demonstrating the importance of combining
SFM with antithetic sampling. In some scenarios, epistemic uncertainty provides
the stronger signal, whereas in other cases aleatoric uncertainty is critical,
indicating that both signals are needed for robust unreliable generation
detection.
%
%
We omit JUMP Unseen Plates because, as shown in
Table~\ref{tab:ood_generalization}, SFM produces reliable generations for
nearly all samples under this mild distribution shift. \rev{Consistently, without filtering, no method separates Unseen Plates inputs from ID inputs (AUROC 0.52--0.54, Table~\ref{tab:ood_detection_nothreshold}): the scores do not flag generations of a model that is robust to this scenario, suggesting that they track unreliable generations rather than scenario novelty alone.}

\paragraph{Ablation Study.}
We further ablate two practical factors in unreliable generation detection:
sampling budget and the construction of unreliable-generation labels. A key
challenge for nested UQ is that finite-sample aleatoric variability can
contaminate epistemic estimates, making the two signals difficult to use
separately. Antithetic sampling reduces this residual Monte Carlo error in
\modelname, which helps separate the two uncertainty components under small
sampling budgets. Table~\ref{tab:sample_size_sensitivity} shows that
the main trends are already preserved with small budgets such as $2\times4$ or
$4\times2$, supporting the sample-efficiency of the estimator. We also test
whether the results depend on the scenario-filtering threshold used to define
unreliable generations. Tables~\ref{tab:ood_detection_nothreshold} and~\ref{tab:ood_detection_meanthreshold} 
show that the relative trends are
broadly stable when using no thresholding or a mean threshold, respectively,
indicating that the conclusions are not driven by a single label-construction
choice. \rev{With a pairing-free MoA-classifier error on BBBC021, \modelname~Epistemic remains the best detector (Appendix~\ref{app:pairfree}).}
Direct equal-budget variance results appear in Appendix~\ref{app:avuq_variance_reduction}.

\section{Conclusion and Limitations}


We showed how SFM improves generalization in distribution-to-distribution
scientific imaging and introduced \modelname~for efficient estimation of
aleatoric and epistemic uncertainty in BSFM using antithetic sampling.
Scalar scores derived from these estimates support unreliable generation
detection. Limitations include approximate posterior inference, numerical
SDE integration, and trace-based covariance scores. \rev{The uncertainty terms serve as ranking statistics: their scale is set by the noise schedule, and calibration is not assessed because it requires paired source--target observations. We evaluate only the shifts studied here, and biological validation is limited to MoA classification.} Future work will evaluate
these components and validate the anomaly-score direction across broader shifts.

\section*{Acknowledgement}

This work was supported in part by ONR Grant N00014-22-1-2110, NSF Grant 2205084, and the Stanford Institute for Human-Centered Artificial Intelligence (HAI). EBF is a Biohub, San Francisco, Investigator. S.Y. is a Chan Zuckerberg Biohub --- San Francisco Investigator.

\bibliographystyle{ieeenat_fullname}
\bibliography{main}

\clearpage
\newcounter{arxivpage}
\setcounter{arxivpage}{\value{page}}
\makeatletter
\g@addto@macro\appendix{\setcounter{page}{\value{arxivpage}}}
\makeatother
\clearpage
\setcounter{page}{1}
\onecolumn
\makeatletter\setlength{\@fptop}{0pt}\makeatother 
\begin{center}
  {\Large\bfseries \thetitle\par}
  \vspace{0.5em}
  {\Large Supplementary Material\par}
  \vspace{1.0em}
\end{center}
\appendix

\section{MAP Approximation for Aleatoric Uncertainty}
\label{app:proof_aleatoric_map}

\begin{proposition}[MAP Approximation for Aleatoric Uncertainty]
\label{prop:aleatoric_map}
Let $V(\theta, \phi) :=
\mathrm{tr}\!\left(\mathrm{Var}(\vx_1 | \vx_0, c, \theta, \phi)\right)$ be
the scalar aleatoric uncertainty score used downstream. Let
$q(\theta,\phi)=\mathcal{N}((\hat{\theta},\hat{\phi}),\mH^{-1})$ be a Laplace
surrogate centered at the MAP estimate. If $V$ is $L$-Lipschitz continuous in
$(\theta,\phi)$ and $\|\mH^{-1}\|_{\mathrm{op}}\leq \epsilon^2$, then
\begin{align}
\left|\mathbb{E}_{q(\theta,\phi)}[V(\theta,\phi)] - V(\hat{\theta}, \hat{\phi})\right|
\leq L\epsilon\sqrt{d},
\end{align}
where $d = \dim(\theta, \phi)$.
\end{proposition}

\begin{remark}
The bound is a surrogate-local statement and can be loose in high dimensions.
It justifies the MAP plug-in as a computationally cheap approximation under
posterior concentration, but it does not imply that MAP is empirically optimal;
in our experiments, nested sampling can provide stronger unreliable generation
detection when the additional SDE solves are affordable. \rev{The proposition assumes a concentrated Gaussian posterior and does not control $L$ or $\mH$ for the actual model; it does not establish calibration of the implemented estimator, reliability under distribution shift, or a relationship between the anomaly score and generation error.}
\end{remark}

\begin{proof}
We prove that the MAP estimate provides a good approximation to the expected scalar aleatoric uncertainty score under the Laplace surrogate stated in Proposition~\ref{prop:aleatoric_map}.

\textbf{Setup.}
Let $V(\theta, \phi) := \mathrm{tr}\!\left(\mathrm{Var}(\vx_1 | \vx_0, c, \theta, \phi)\right)$ denote the scalar conditional uncertainty score as a function of parameters.
Let $q(\theta,\phi)=\mathcal{N}((\hat{\theta},\hat{\phi}),\mH^{-1})$, where $(\hat{\theta}, \hat{\phi})$ is the MAP estimate and $\mH$ is the Hessian of the negative log-posterior at the MAP.

\textbf{Step 1: Apply Lipschitz continuity.}
By the triangle inequality and Lipschitz continuity of $V$:
\begin{align}
&\left|\mathbb{E}_{q(\theta,\phi)}[V(\theta, \phi)] - V(\hat{\theta}, \hat{\phi})\right| \nonumber \\
&= \left|\mathbb{E}_{q(\theta,\phi)}[V(\theta, \phi) - V(\hat{\theta}, \hat{\phi})]\right| \nonumber \\
&\leq \mathbb{E}_{q(\theta,\phi)}[|V(\theta, \phi) - V(\hat{\theta}, \hat{\phi})|] \nonumber \\
&\leq L \cdot \mathbb{E}_{q(\theta,\phi)}[\|(\theta, \phi) - (\hat{\theta}, \hat{\phi})\|_2],
\end{align}
where the last inequality uses the $L$-Lipschitz property:
$|V(\theta, \phi) - V(\theta', \phi')| \leq L\|(\theta, \phi) - (\theta', \phi')\|_2$.

\textbf{Step 2: Bound the expected deviation.}
Let $\delta := (\theta, \phi) - (\hat{\theta}, \hat{\phi}) \sim \mathcal{N}(\vzero, \mH^{-1})$ denote the parameter deviation from the MAP.
By Jensen's inequality (since $\|\cdot\|_2$ is convex):
\begin{align}
\mathbb{E}[\|\delta\|_2]
&= \mathbb{E}\left[\sqrt{\sum_{i=1}^d \delta_i^2}\right]
\leq \sqrt{\mathbb{E}\left[\sum_{i=1}^d \delta_i^2\right]}
= \sqrt{\text{tr}(\mH^{-1})},
\end{align}
where $d = \dim(\theta, \phi)$ is the total parameter dimension.

\textbf{Step 3: Apply the operator norm bound.}
Using the relationship between trace and operator norm:
\begin{align}
\text{tr}(\mH^{-1})
= \sum_{i=1}^d \lambda_i(\mH^{-1})
\leq d \cdot \max_i \lambda_i(\mH^{-1})
= d \cdot \|\mH^{-1}\|_{\text{op}},
\end{align}
where $\lambda_i(\mH^{-1})$ are the eigenvalues of $\mH^{-1}$ and $\|\mH^{-1}\|_{\text{op}}$ is the operator norm.

\textbf{Step 4: Combine the bounds.}
By assumption, $\|\mH^{-1}\|_{\text{op}} \leq \epsilon^2$. Therefore:
\begin{align}
\mathbb{E}[\|\delta\|_2]
\leq \sqrt{\text{tr}(\mH^{-1})}
\leq \sqrt{d \cdot \|\mH^{-1}\|_{\text{op}}}
\leq \sqrt{d \cdot \epsilon^2}
= \epsilon\sqrt{d}.
\end{align}

\textbf{Step 5: Final bound.}
Combining Steps 1 and 4:
\begin{align}
\left|\mathbb{E}_{q(\theta,\phi)}[V(\theta, \phi)] - V(\hat{\theta}, \hat{\phi})\right|
\leq L \cdot \mathbb{E}[\|\delta\|_2]
\leq L\epsilon\sqrt{d}.
\end{align}

\end{proof}

\section{Dataset and OOD Scenario Details}
\label{app:dataset}
\paragraph{Cell Imaging Datasets.}
\textbf{BBBC021.} The BBBC021v1 dataset \citep{caie2010high}
from the Broad Bioimage Benchmark Collection
is a benchmark for
image-based phenotypic profiling of chemical
perturbations in MCF-7 breast cancer cells.
It contains 97,504 fluorescent microscopy images
captured from cells treated with 113 small molecules
at eight concentrations, targeting diverse cellular
mechanisms including actin disruption, Aurora kinase
inhibition, and microtubule stabilization. Each
image provides multi-channel fluorescence for DNA,
F-actin, and beta-tubulin, enabling detailed
morphological analysis.

\begin{table}[t]
\centering
\small
\caption{Summary of datasets and OOD scenarios.}
\begin{tabular}{l|llllll}
\toprule
Dataset & BBBC021 & BBBC021 & JUMP & JUMP & JUMP & fMRI\\
OOD Case & Unseen Pert. & Intensity Shift & Unseen Cell Lines & Unseen Plates & Intensity Shift & Low ToM\\
\midrule
\# Channels & 3 & 3 & 5 & 5 & 5 & 1\\
\# Images & 98K & 98K & 72K & 72K & 72K & 12K \\
Dist Shift & Extreme & High & Medium & Low & High & -\\
\bottomrule
\end{tabular}
\label{tab:dataset_summary}
\end{table}

\textbf{JUMP.}
The JUMP dataset \citep{chandrasekaran2023jump} is the most
comprehensive image-based profiling resource to date, integrating
both genetic and chemical perturbations. It comprises approximately
3 million images capturing phenotypic responses of 75 million
single cells to genetic knockouts (CRISPR/ORF) and chemical
treatments. The dataset includes two cell types: U2OS and A549.
We use A549 cells for model training and U2OS cells for OOD
evaluation. For our experiments, we use a subset of 72,000
images and focus exclusively on chemical perturbations.

\paragraph{fMRI Dataset.}
We use the ds000228-1.1.1 dataset \citep{richardson2018development}
containing MRI data from 3--12-year-old children viewing a Pixar animated film. The task is to
transform brain states from resting baseline (TRs 1--10, 20 seconds) to Theory of Mind (ToM) activation
(7 events totaling 25 TRs, 50 seconds). We focus on slices 13--16 (4 middle brain slices capturing
ToM-responsive regions: temporoparietal junction and medial prefrontal cortex), with each slice as a 
64$\times$64 grayscale image. The dataset includes 122 children categorized by False Belief performance: PASS
group (84 subjects, mean ToM score: $0.883\pm0.102$, age: $7.6\pm2.2$ years) as in-domain, and
INC+FAIL groups (38 subjects, mean ToM score: $0.537\pm0.148$, age: $4.7\pm1.0$ years) as out-of-distribution.
Each subject contributes 40 rest images and 100 ToM images. The training set contains 9,712
images from PASS subjects; the test ID set contains 2,048 randomly selected PASS images; and the test OOD
set contains 2,048 randomly selected INC+FAIL images, all normalized per slice.

\rev{\textbf{Why this dataset.} We include fMRI to test whether the approach carries over to a second imaging modality with a subject-level cohort shift. ds000228 provides three properties that this test requires: (i) source and target states within the same scan, namely an annotated baseline period and annotated ToM events, so that source and target come from the same subject, analogous to control and perturbed cells; (ii) an independent behavioral measure (false-belief task performance) that defines the cohorts, so the shift is set by a non-imaging variable rather than by acquisition; and (iii) a movie-viewing paradigm suitable for young children, so the OOD cohort is also younger (4.7 vs.\ 7.6 years), an age range over which ToM-network responses become increasingly specialized \citep{richardson2018development}. Many cohort-shift datasets in medical imaging provide a single image type per subject and therefore do not define a source-to-target task. The images are single-TR BOLD slices; BOLD contrast is dominated by static tissue contrast, and task-evoked changes are only a few percent of the signal, so rest and ToM frames look similar by eye.}

\section{Implementation Details}
\label{app:implementation}
\paragraph{Filtering Strategy for Unreliable Generation Detection.}
To ensure the reliability of our unreliable generation detection benchmarks, we filter out ambiguous
samples that may not represent true in-distribution (ID) or out-of-distribution
(OOD) behavior. For BBBC021, we compute the feature-space distance between model
predictions and randomly paired ground-truth treatment images using a
pretrained mode-of-action classifier. For JUMP, we use the Structural
Similarity Index Measure (SSIM) as the distance metric. We apply thresholds
based on the distances calculated across the dataset. Specifically, we remove ID
samples with distances greater than $\mu - 0.5\sigma$ (potentially poor generations)
and remove OOD samples with distances smaller than $\mu + 0.5\sigma$ (samples that
appear misleadingly in-distribution), where $\mu$ and $\sigma$ denote the mean and
standard deviation of the distance metric, respectively.

\paragraph{\rev{Positive rates.}} \rev{AUPR depends on the fraction of positive (unreliable) samples, so it is comparable across methods within a cohort but not across cohorts. The positive rate, which equals the AUPR of a random detector, is 0.28, 0.42, 0.42, and 0.45 for BBBC021 Unseen Pert., BBBC021 Intensity Shift, JUMP Unseen Cell Lines, and JUMP Intensity Shift in Table~\ref{tab:unreliable_generation_detection} (0.40 for JUMP Unseen Plates in Tables~\ref{tab:sample_size_sensitivity} and~\ref{tab:jump_plate_detection}); 0.35 for BBBC021 Unseen Pert.\ and 0.50 for the other cohorts without filtering (Table~\ref{tab:ood_detection_nothreshold}); and 0.26, 0.40, 0.42, 0.42, and 0.44 for the five cohorts with the mean threshold (Table~\ref{tab:ood_detection_meanthreshold}). These are the rates of the \modelname~evaluation sets; those of the other methods differ by at most 0.06.}

\section{Unreliable Generation Detection Sensitivity Analyses}
\label{app:ood_sensitivity}

Table~\ref{tab:sample_size_sensitivity} reports sensitivity to the number of posterior samples and SDE samples used by \modelname. The results show that even small sampling budgets, such as $2\times4$ or $4\times2$, already preserve the main unreliable generation detection trends, supporting the sample-efficiency role of antithetic sampling.

\begin{table*}[t]
\centering
\small
\caption{Sample-size sensitivity for \modelname. Each row reports AUROC/AUPR for a posterior-sample $\times$ SDE-sample budget.}
\resizebox{\textwidth}{!}{
\begin{tabular}{lcccccccccc}
\toprule
& \multicolumn{4}{c}{\textbf{BBBC021}} & \multicolumn{6}{c}{\textbf{JUMP}} \\
\cmidrule(lr){2-5} \cmidrule(lr){6-11}
\textbf{Method} & \multicolumn{2}{c}{Unseen Pert.} & \multicolumn{2}{c}{Intensity Shift} & \multicolumn{2}{c}{Unseen Cell Lines} & \multicolumn{2}{c}{Unseen Plates} & \multicolumn{2}{c}{Intensity Shift}\\
\cmidrule(lr){2-3} \cmidrule(lr){4-5} \cmidrule(lr){6-7} \cmidrule(lr){8-9} \cmidrule(lr){10-11}
& AUROC$\uparrow$ & AUPR$\uparrow$ & AUROC$\uparrow$ & AUPR$\uparrow$ & AUROC$\uparrow$ & AUPR$\uparrow$ & AUROC$\uparrow$ & AUPR$\uparrow$ & AUROC$\uparrow$ & AUPR$\uparrow$\\
\midrule
2x2~Aleatoric & 0.673 & 0.404 & 0.679 & 0.572 & \textbf{0.838} & \textbf{0.792} & \textbf{0.735} & \textbf{0.662} & 0.303 & 0.331\\
2x2~Epistemic & \textbf{0.752} & \textbf{0.551} & \textbf{0.785} & \textbf{0.715} & 0.743 & 0.678 & 0.628 & 0.595 & \textbf{0.826} & \textbf{0.710}\\
\midrule
2x4~Aleatoric & 0.718 & 0.469 & 0.715 & 0.647 & \textbf{0.852} & \textbf{0.818} & \textbf{0.747} & \textbf{0.673} & 0.363 & 0.349\\
2x4~Epistemic & \textbf{0.725} & \textbf{0.597} & \textbf{0.784} & \textbf{0.763} & 0.759 & 0.709 & 0.614 & 0.542 & \textbf{0.846} & \textbf{0.737}\\
\midrule
4x2~Aleatoric & 0.694 & 0.426 & 0.688 & 0.598 & \textbf{0.856} & \textbf{0.813} & \textbf{0.743} & \textbf{0.669} & 0.278 & 0.323\\
4x2~Epistemic & \textbf{0.773} & \textbf{0.621} & \textbf{0.784} & \textbf{0.736} & 0.779 & 0.709 & 0.671 & 0.594 & \textbf{0.891} & \textbf{0.801}\\
\midrule
4x4~Aleatoric & 0.730 & 0.486 & 0.720 & 0.649 & \textbf{0.860} & \textbf{0.829} & \textbf{0.739} & \textbf{0.666} & 0.364 & 0.350\\
4x4~Epistemic & \textbf{0.788} & \textbf{0.648} & \textbf{0.795} & \textbf{0.766} & 0.800 & 0.738 & 0.667 & 0.578 & \textbf{0.897} & \textbf{0.816}\\
\midrule
8x4~Aleatoric & 0.730 & 0.500 & 0.709 & 0.643 & \textbf{0.870} & \textbf{0.843} & \textbf{0.744} & \textbf{0.676} & 0.357 & 0.347\\
8x4~Epistemic & \textbf{0.807} & \textbf{0.672} & \textbf{0.770} & \textbf{0.750} & 0.803 & 0.745 & 0.679 & 0.599 & \textbf{0.906} & \textbf{0.831}\\
\bottomrule
\end{tabular}
}
\label{tab:sample_size_sensitivity}
\end{table*}

Tables~\ref{tab:ood_detection_nothreshold} and~\ref{tab:ood_detection_meanthreshold} report unreliable generation detection results under alternative unreliable-generation label construction choices. The no-threshold setting uses scenario provenance directly, while the mean-threshold setting uses $\mu$ rather than $\mu+0.5\sigma$ for high-error OOD filtering. The relative trends are broadly preserved across these alternatives, suggesting that the headline results are not driven by a single task-specific threshold.

\begin{table*}[t]
\centering
\small
\caption{Unreliable generation detection results without thresholding.}
\resizebox{\textwidth}{!}{
\begin{tabular}{lcccccccccc}
\toprule
& \multicolumn{4}{c}{\textbf{BBBC021}} & \multicolumn{6}{c}{\textbf{JUMP}} \\
\cmidrule(lr){2-5} \cmidrule(lr){6-11}
\textbf{Method} & \multicolumn{2}{c}{Unseen Pert.} & \multicolumn{2}{c}{Intensity Shift} & \multicolumn{2}{c}{Unseen Cell Lines} & \multicolumn{2}{c}{Unseen Plates} & \multicolumn{2}{c}{Intensity Shift}\\
\cmidrule(lr){2-3} \cmidrule(lr){4-5} \cmidrule(lr){6-7} \cmidrule(lr){8-9} \cmidrule(lr){10-11}
& AUROC$\uparrow$ & AUPR$\uparrow$ & AUROC$\uparrow$ & AUPR$\uparrow$ & AUROC$\uparrow$ & AUPR$\uparrow$ & AUROC$\uparrow$ & AUPR$\uparrow$ & AUROC$\uparrow$ & AUPR$\uparrow$\\
\midrule
MAP Aleatoric & 0.616 & 0.443 & 0.636 & 0.626 & 0.670 & 0.627 & 0.517 & 0.519 & 0.227 & 0.351 \\
SWAG Aleatoric & 0.684 & 0.589 & 0.680 & 0.729 & 0.699 & 0.647 & 0.522 & 0.520 & 0.813 & 0.712 \\
SWAG Epistemic & 0.661 & 0.572 & 0.680 & 0.730 & 0.689 & 0.638 & 0.516 & 0.518 & 0.806 & 0.702 \\
LA Aleatoric & 0.662 & 0.574 & 0.671 & 0.727 & 0.684 & 0.646 & 0.527 & 0.528 & 0.804 & 0.707 \\
LA Epistemic & 0.635 & 0.495 & 0.669 & 0.713 & 0.646 & 0.615 & \textbf{0.540} & \textbf{0.543} & 0.512 & 0.466 \\
\midrule
MCD Aleatoric & 0.687 & 0.516 & 0.691 & 0.681 & 0.716 & 0.674 & 0.522 & 0.522 & 0.300 & 0.372 \\
MCD Epistemic & 0.693 & 0.507 & 0.692 & 0.673 & 0.686 & 0.649 & 0.521 & 0.519 & 0.299 & 0.372 \\
\modelname~Aleatoric & 0.672 & 0.502 & 0.690 & 0.679 & \textbf{0.718} & \textbf{0.676} & 0.524 & 0.523 & 0.224 & 0.350 \\
\modelname~Epistemic & \textbf{0.764} & \textbf{0.695} & \textbf{0.752} & \textbf{0.784} & 0.695 & 0.647 & 0.535 & 0.532 & \textbf{0.829} & \textbf{0.721} \\
\bottomrule
\end{tabular}
}
\label{tab:ood_detection_nothreshold}
\end{table*}

\begin{table*}[t]
\centering
\small
\caption{Unreliable generation detection results with mean threshold $\mu$.}
\resizebox{\textwidth}{!}{
\begin{tabular}{lcccccccccc}
\toprule
& \multicolumn{4}{c}{\textbf{BBBC021}} & \multicolumn{6}{c}{\textbf{JUMP}} \\
\cmidrule(lr){2-5} \cmidrule(lr){6-11}
\textbf{Method} & \multicolumn{2}{c}{Unseen Pert.} & \multicolumn{2}{c}{Intensity Shift} & \multicolumn{2}{c}{Unseen Cell Lines} & \multicolumn{2}{c}{Unseen Plates} & \multicolumn{2}{c}{Intensity Shift}\\
\cmidrule(lr){2-3} \cmidrule(lr){4-5} \cmidrule(lr){6-7} \cmidrule(lr){8-9} \cmidrule(lr){10-11}
& AUROC$\uparrow$ & AUPR$\uparrow$ & AUROC$\uparrow$ & AUPR$\uparrow$ & AUROC$\uparrow$ & AUPR$\uparrow$ & AUROC$\uparrow$ & AUPR$\uparrow$ & AUROC$\uparrow$ & AUPR$\uparrow$\\
\midrule
MAP Aleatoric & 0.639 & 0.375 & 0.667 & 0.547 & 0.767 & 0.698 & 0.638 & 0.574 & 0.286 & 0.319\\
SWAG Aleatoric & 0.701 & 0.550 & 0.670 & 0.683 & 0.794 & 0.728 & 0.615 & 0.562 & 0.813 & 0.712\\
SWAG Epistemic & 0.676 & 0.523 & 0.678 & 0.684 & 0.782 & 0.714 & 0.627 & 0.574 & 0.879 & 0.797\\
LA Aleatoric & 0.674 & 0.534 & 0.654 & 0.678 & 0.779 & 0.721 & 0.630 & 0.583 & 0.875 & 0.793\\
LA Epistemic & 0.640 & 0.435 & 0.645 & 0.649 & 0.739 & 0.682 & 0.628 & 0.600 & 0.564 & 0.454\\
MCD Aleatoric & 0.723 & 0.461 & 0.700 & 0.591 & 0.829 & 0.780 & 0.651 & 0.599 & 0.381 & 0.349\\
MCD Epistemic & 0.727 & 0.453 & 0.706 & 0.595 & 0.777 & 0.728 & 0.634 & 0.579 & 0.371 & 0.346\\
\midrule
\modelname~Aleatoric & 0.710 & 0.467 & 0.699 & 0.602 & 0.841 & 0.796 & 0.666 & 0.614 & 0.302 & 0.324\\
\modelname~Epistemic & 0.795 & 0.633 & 0.764 & 0.716 & 0.784 & 0.710 & 0.614 & 0.557 & 0.893 & 0.802\\
\bottomrule
\end{tabular}
}
\label{tab:ood_detection_meanthreshold}
\end{table*}

\section{Responsible Use Safeguards}
\label{app:responsible_use}

The proposed methods are intended to support scientific imaging analysis by
improving generalization and flagging potentially unreliable generations. The
generated images and uncertainty scores should be used for hypothesis
prioritization and model diagnostics, not as substitutes for biological or
clinical validation. Any released code or model artifacts should be accompanied
by documentation describing the intended research use, dataset provenance,
evaluation settings, and the need for independent validation before drawing
therapeutic, biological, or clinical conclusions.

\section{Additional Experimental Results}

\subsection{Controlled SFM Training and Sampling Ablation}
\label{app:sfm_component_ablation}

Table~\ref{tab:sfm_component_ablation} separates stochastic-interpolant
training from the inference sampler. The deterministic FM baseline was
trained from scratch with stochastic noise set to zero and without a noise
head; it is not a pretrained CellFlux checkpoint. Deterministic FM and SFM
use the same training data, preprocessing, U-Net trunk, optimizer, batch
size, and convergence criterion. The deterministic and SFM models have
56,201,219 and 56,208,646 parameters on BBBC021, respectively, and
56,312,837 and 56,324,874 on JUMP, respectively. The noise head therefore
adds 0.013\% and 0.021\% parameters. For the sampler comparison, we use
fixed trained checkpoints, identical evaluation cohorts, integration steps,
and numbers of generated images. ``Naive SDE'' adds inference noise without
score correction; SFM with ODE sampling retains stochastic-interpolant
training but omits stochastic sampling.

\begin{table}[H]
\centering
\small
\caption{Controlled training and inference comparison. Entries are FID / KID
(lower is better). Bold marks the best result among these five variants for
each cohort. The results use fixed trained checkpoints rather than multiple
independent training runs.}
\resizebox{\textwidth}{!}{%
\begin{tabular}{lcccc}
\toprule
Training / inference & BBBC021 Unseen Pert. & BBBC021 Intensity Shift & JUMP Unseen Cell Lines & JUMP Unseen Plates \\
\midrule
Deterministic FM + ODE (CellFlux) & 103.73 / 12.76 & 62.01 / 5.84 & 39.06 / 3.40 & \textbf{17.81 / 1.04} \\
Deterministic FM + naive SDE & 170.74 / 18.42 & 183.46 / 20.89 & 249.22 / 29.95 & 201.50 / 23.34 \\
SFM + ODE & 36.81 / 3.24 & 42.12 / 3.77 & 25.73 / 1.98 & 22.32 / 1.65 \\
SFM + naive SDE & 117.14 / 11.00 & 129.41 / 12.06 & 288.73 / 34.37 & 247.49 / 28.63 \\
SFM + score-corrected SDE & \textbf{33.29 / 2.02} & \textbf{28.14 / 1.87} & \textbf{25.10 / 1.75} & 18.02 / 1.23 \\
\bottomrule
\end{tabular}%
}
\label{tab:sfm_component_ablation}
\end{table}

Naive inference noise degrades both trained backbones, whereas the
score-corrected SDE improves on SFM with ODE sampling in all four cohorts.
The SFM ODE variant improves on the matched deterministic FM ODE baseline
in three of the four cohorts.

\subsection{Direct Variance Reduction with Antithetic Sampling}
\label{app:avuq_variance_reduction}

We compare i.i.d. and antithetic SDE samples at the same budget of $M=4$
posterior draws and $K=4$ SDE samples per draw. Table~\ref{tab:avuq_variance_reduction}
reports antithetic-pair correlation and the variance of the estimated
conditional mean across three sampling seeds. The variance-reduction factor
(VRF) is the i.i.d. variance divided by the antithetic variance. Pair
correlations are negative in all four cohorts, and the variance falls by
1.82--4.29 times at the same 16-sample budget.

\begin{table}[H]
\centering
\small
\caption{Equal-budget i.i.d. versus antithetic SDE sampling ($M=4$, $K=4$).
Conditional-mean estimator variance is measured across three sampling seeds;
VRF is i.i.d. variance divided by antithetic variance.}
\begin{tabular}{lcccc}
\toprule
Shift & Pair corr. $\rho$ & i.i.d. variance & Antithetic variance & VRF \\
\midrule
BBBC021 unseen perturbations & $-0.774$ & $1.70\times10^{-3}$ & $6.44\times10^{-4}$ & $2.64\times$ \\
BBBC021 intensity shift & $-0.773$ & $1.70\times10^{-3}$ & $7.06\times10^{-4}$ & $2.41\times$ \\
JUMP unseen cell lines & $-0.524$ & $8.06\times10^{-4}$ & $4.43\times10^{-4}$ & $1.82\times$ \\
JUMP intensity shift & $-0.748$ & $1.40\times10^{-3}$ & $3.26\times10^{-4}$ & $4.29\times$ \\
\bottomrule
\end{tabular}
\label{tab:avuq_variance_reduction}
\end{table}

\paragraph{\rev{Residual-corrected estimators at equal budget.}} \rev{For i.i.d.\ paths, $\mathbb{E}[\hat{E}]=E+A/K$ with $A=\mathbb{E}_{p(\theta,\phi|\gD)}[\mathrm{Var}(\vx_1|\vx_0,c,\theta,\phi)]$ (\eqref{eq:MCnoise}); with antithetic pairs of correlation $\rho$, the residual becomes $(1+\rho)A/K$, i.e., 0.23--0.48 of the i.i.d.\ residual for the correlations in Table~\ref{tab:avuq_variance_reduction}, at no extra cost. We also subtracted the estimated residual from each estimator at $M=K=4$: $\hat{A}/K$ for i.i.d.\ sampling and $\hat{A}_{\mathrm{pair}}/J$ for antithetic sampling, where $\hat{A}_{\mathrm{pair}}$ is the variance of the $J=K/2$ pair means. At this budget, the estimated residual is of the same order as $\hat{E}$ for both designs, so the corrected per-input estimates are dominated by sampling noise: about half are negative, and with the score and labels of Table~\ref{tab:ood_detection_nothreshold} their AUROC is close to chance (0.46--0.48 for both designs in all four cohorts of Table~\ref{tab:unreliable_generation_detection}). Subtracting $\hat{A}/K$ is the ideal correction, but it becomes low-variance only for large $M$ and $K$; \modelname~instead reduces the residual at a fixed budget.}

\subsection{Comparison with Unsupervised OOD Detectors}
\label{app:unsupervised_detectors}

Table~\ref{tab:unsupervised_detector_comparison} compares \modelname\ with
ODEED, a probability-flow-ODE reconstruction detector
\citep{le2024detecting}, and PatchCore, a nominal patch-feature detector
\citep{roth2022towards}. The detectors were fit using only in-domain data;
neither OOD examples nor unreliable-generation labels were used for fitting.
The labels are used only to evaluate AUROC and AUPR. All entries in this
comparison are point estimates; the ODEED and PatchCore rows do not imply
the three-seed standard errors reported for methods in the main detection
table. The stronger \modelname\ component depends on the shift: epistemic
for both BBBC021 shifts and JUMP intensity shift, and aleatoric for JUMP
unseen cell lines.

\begin{table}[H]
\centering
\small
\caption{Unsupervised detector comparison. Each entry is AUROC / AUPR
(higher is better); bold marks the best point estimate in each cohort.}
\resizebox{\textwidth}{!}{%
\begin{tabular}{lcccc}
\toprule
Method & BBBC021 Unseen Pert. & BBBC021 Intensity Shift & JUMP Unseen Cell Lines & JUMP Intensity Shift \\
\midrule
ODEED & 0.615 / 0.505 & 0.631 / 0.671 & 0.686 / 0.564 & 0.543 / 0.568 \\
PatchCore & 0.686 / 0.481 & 0.751 / 0.709 & 0.594 / 0.545 & 0.866 / 0.753 \\
\modelname\ Aleatoric & 0.742 / 0.523 & 0.729 / 0.663 & \textbf{0.869 / 0.835} & 0.309 / 0.335 \\
\modelname\ Epistemic & \textbf{0.809 / 0.698} & \textbf{0.785 / 0.768} & 0.794 / 0.726 & \textbf{0.890 / 0.799} \\
\bottomrule
\end{tabular}%
}
\label{tab:unsupervised_detector_comparison}
\end{table}

\subsection{JUMP Unseen-Plate Detection}
\label{app:jump_plate_detection}

JUMP unseen plates is a relatively mild shift for SFM and yields fewer
unreliable generations, so the main detection table focuses on the four
more failure-rich cohorts. Table~\ref{tab:jump_plate_detection} reports
the available unseen-plate detection point estimates. \modelname\ Aleatoric
has the highest AUROC, while LA Aleatoric has the highest AUPR.

\begin{table}[H]
\centering
\small
\caption{JUMP unseen-plate unreliable-generation detection. Point estimates
are AUROC and AUPR (higher is better).}
\begin{tabular}{lcc}
\toprule
Method & AUROC & AUPR \\
\midrule
MAP Aleatoric & 0.689 & 0.607 \\
MCD-DFM Epistemic & 0.571 & 0.509 \\
SWAG Aleatoric & 0.728 & 0.641 \\
SWAG Epistemic & 0.727 & 0.657 \\
LA Aleatoric & 0.732 & \textbf{0.670} \\
LA Epistemic & 0.697 & 0.629 \\
MCD Aleatoric & 0.734 & 0.657 \\
MCD Epistemic & 0.719 & 0.647 \\
\modelname\ Aleatoric & \textbf{0.737} & 0.661 \\
\modelname\ Epistemic & 0.665 & 0.580 \\
\bottomrule
\end{tabular}
\label{tab:jump_plate_detection}
\end{table}

\subsection{Mechanism-of-Action Classification of Generated Cell Images}
\label{app:moa_classification}

Table~\ref{tab:moa_classification} evaluates top-1 mechanism-of-action
(MoA) accuracy using a pretrained BBBC021 classifier (13 classes; 7.69\%
uniform chance). Its 95\% intervals use 2,000 perturbation-stratified
bootstrap resamples. The columns separate intensity shift from its
combination with unseen perturbations; real target images give a reference.
SFM leads the generators in both cohorts, but its interval overlaps the
closest baselines only under the combined shift.

\begin{table}[H]
\centering
\small
\caption{Top-1 MoA accuracy (95\% bootstrap confidence interval). Higher is
better. The same pretrained classifier is evaluated on generated and real
target images.}
\resizebox{0.7\textwidth}{!}{%
\begin{tabular}{lcc}
\toprule
Method & Intensity shift & Intensity shift + unseen perturbations \\
\midrule
Uniform chance (13 classes) & 7.69\% & 7.69\% \\
BBDM & 7.32\% (6.49--8.15\%) & 9.93\% (8.33--11.61\%) \\
GOUB & 20.13\% (18.90--21.36\%) & 3.20\% (2.24--4.16\%) \\
UniDB & 25.62\% (24.36--26.97\%) & 5.68\% (4.40--7.05\%) \\
UNSB & 4.52\% (3.92--5.20\%) & 10.61\% (9.01--12.45\%) \\
SDEdit & 43.38\% (42.07--44.64\%) & 5.28\% (4.24--6.41\%) \\
CellFlux & 36.43\% (35.14--37.63\%) & 10.13\% (8.53--11.73\%) \\
SFM & \textbf{48.47\% (46.93--49.96\%)} & \textbf{12.89\% (11.13--14.73\%)} \\
\midrule
Real target images & 66.37\% (64.97--67.72\%) & 43.23\% (40.59--45.96\%) \\
\bottomrule
\end{tabular}%
}
\label{tab:moa_classification}
\end{table}

\subsection{Visualizing SFM Generalization}

Figures~\ref{fig:ood_intensity_uperb}, \ref{fig:ood_intensity},
\ref{fig:ood_cellline}, and~\ref{fig:ood_plate}
show additional examples of generated images from different methods on
BBBC021 and JUMP under various OOD scenarios,
compared with the source images and ground-truth target images.

\begin{figure*}[t!]
    \centering
    \includegraphics[width=0.6\linewidth]{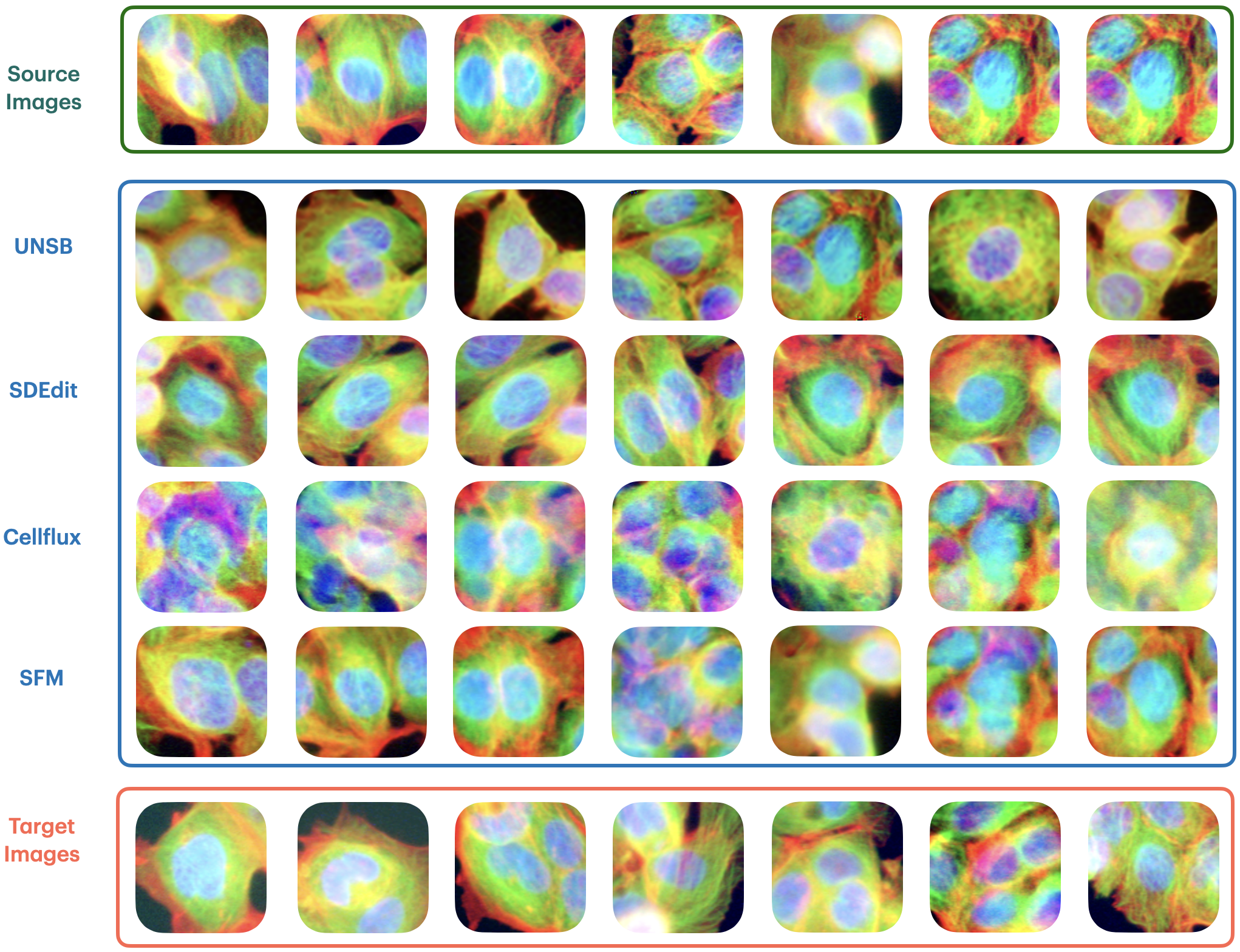}
    \caption{Examples of generated images from different methods on
    BBBC021 under the Unseen Pert. OOD scenario,
    compared with source images and ground-truth target images.}
    \label{fig:ood_intensity_uperb}
\end{figure*}

\begin{figure*}[t!]
    \centering
    \includegraphics[width=0.6\linewidth]{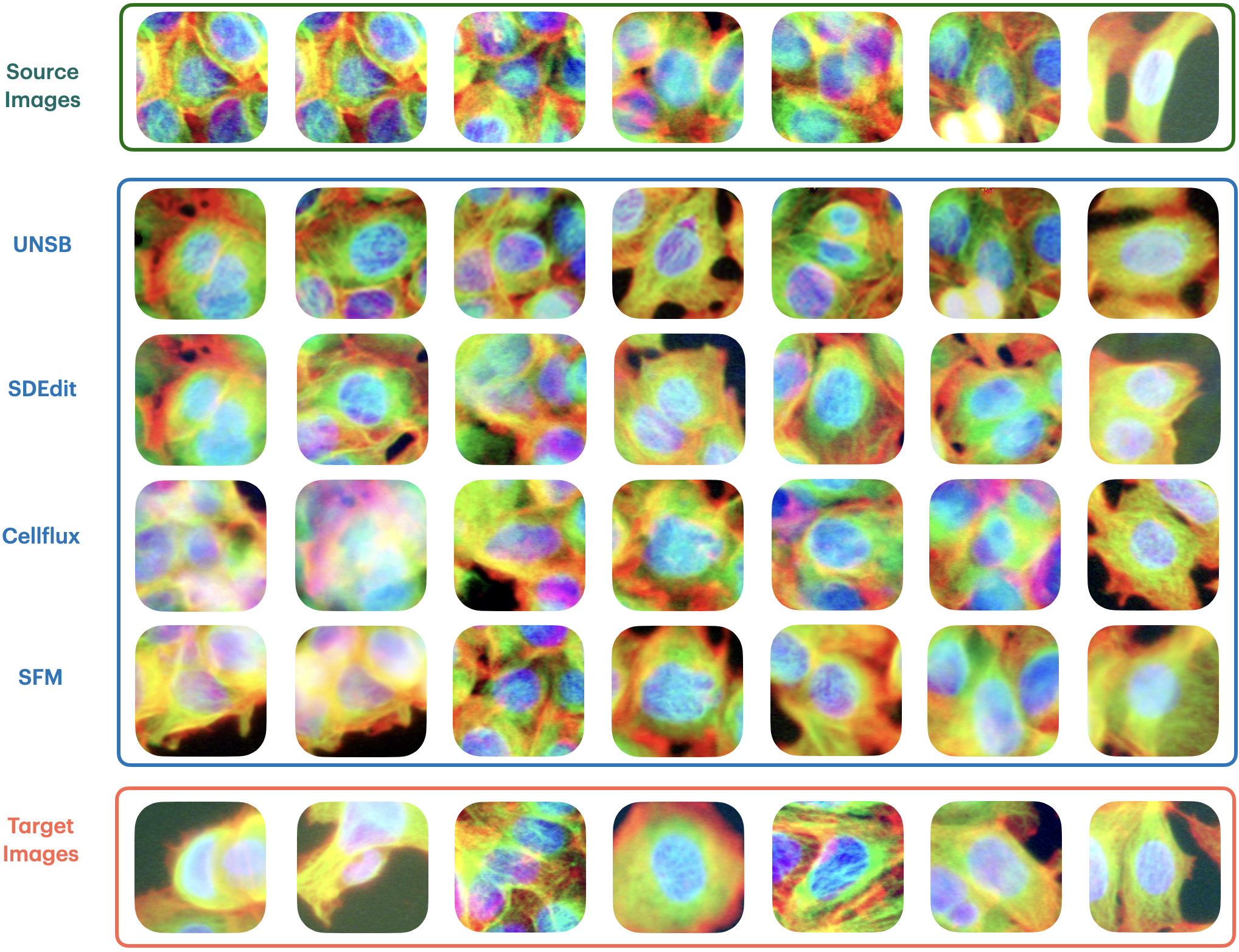}
    \caption{Examples of generated images from different methods on
    BBBC021 under the Intensity Shift OOD scenario,
    compared with source images and ground-truth target images.}
    \label{fig:ood_intensity}
\end{figure*}

\begin{figure*}[t!]
    \centering
    \includegraphics[width=0.6\linewidth]{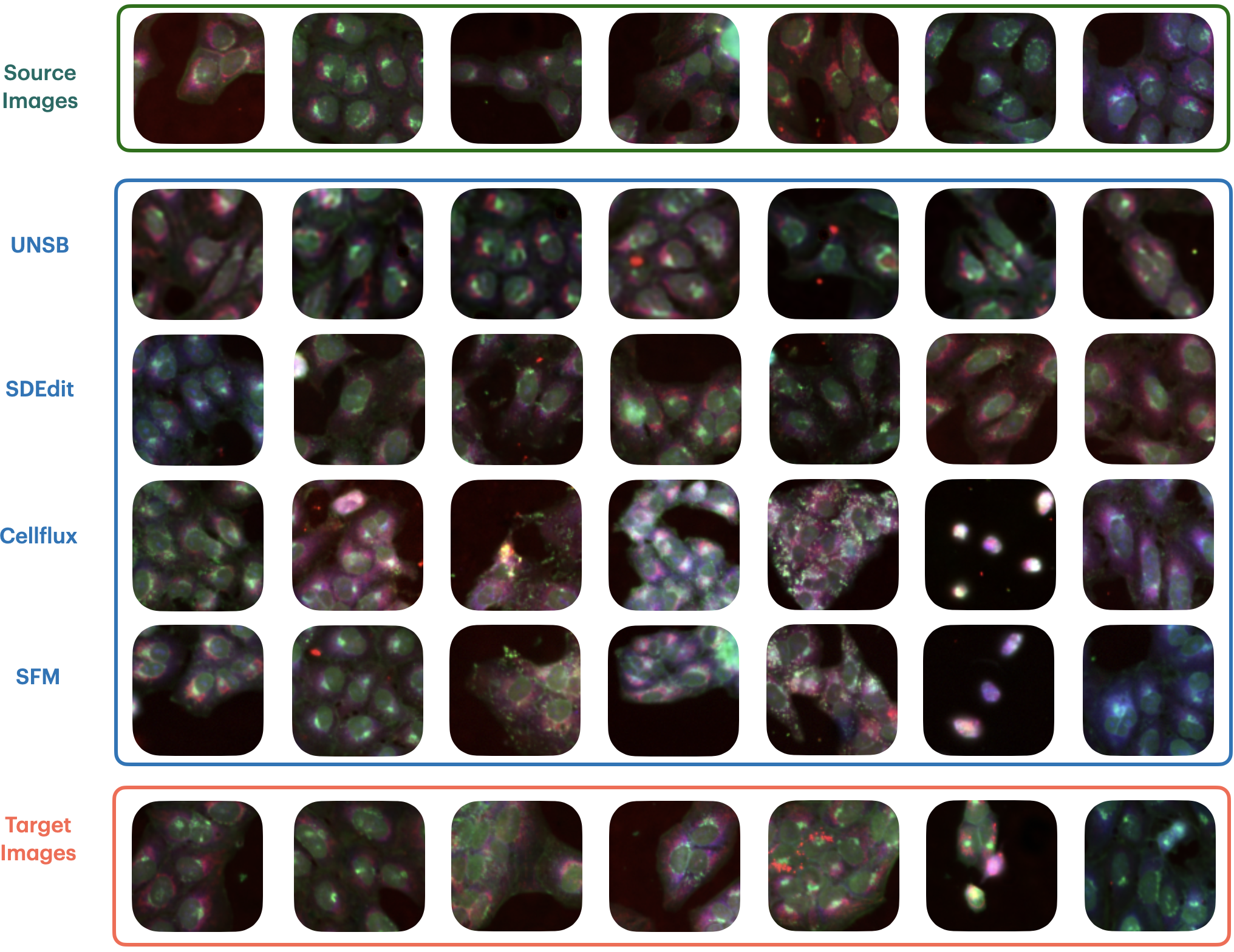}
    \caption{Examples of generated images from different methods on
    JUMP under the Unseen Cell Lines OOD scenario,
    compared with source images and ground-truth target images.}
    \label{fig:ood_cellline}
\end{figure*}

\begin{figure*}[t!]
    \centering
    \includegraphics[width=0.6\linewidth]{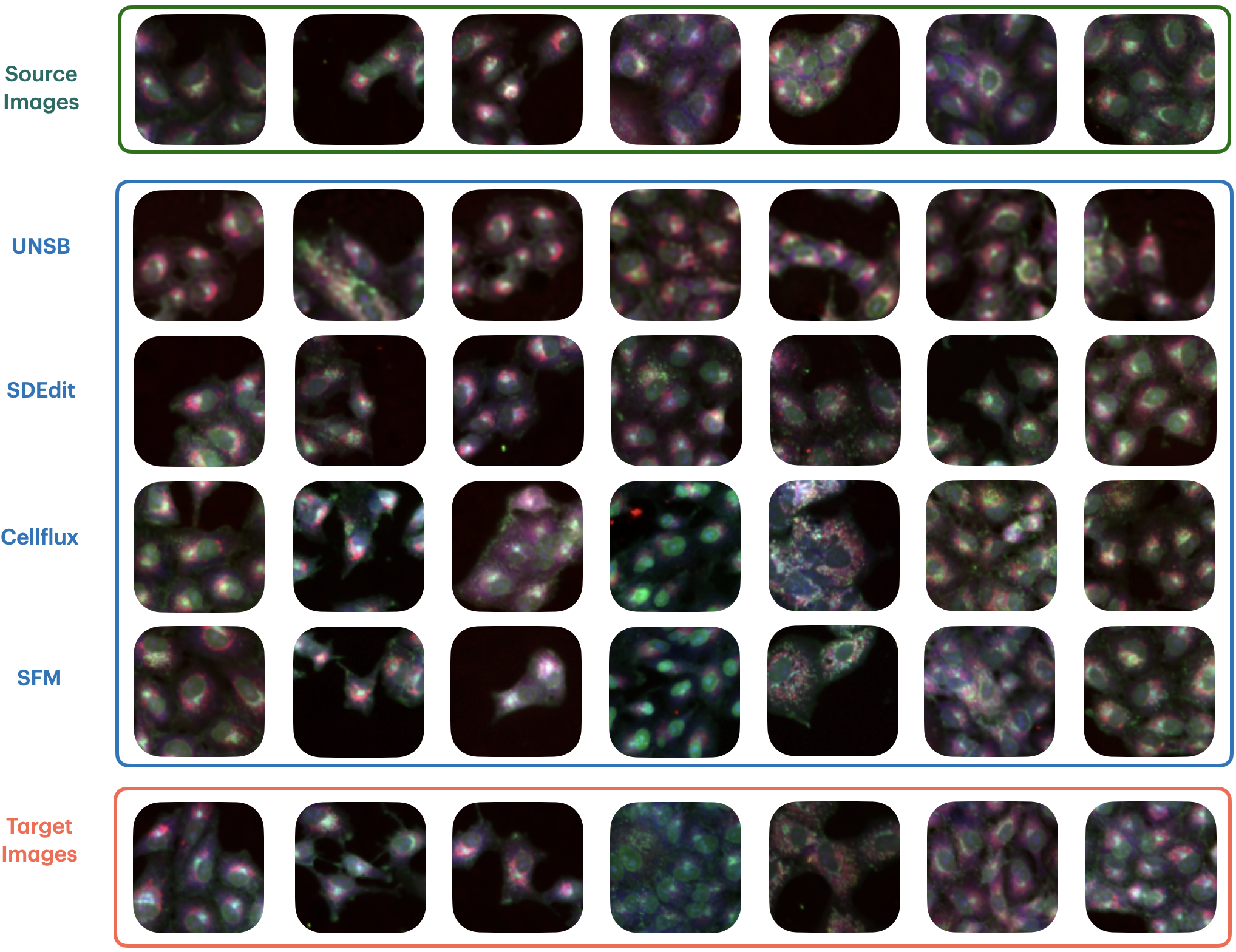}
    \caption{Examples of generated images from different methods on
    JUMP under the Unseen Plates OOD scenario,
    compared with source images and ground-truth target images.}
    \label{fig:ood_plate}
\end{figure*}

{\color{blue}
\section{Marginal-Preserving SDE: Derivation and Implementation Details}
\label{app:sde_derivation}

\subsection{Derivation of the Marginal-Preserving Drift}
Fix $c$ and write $p_t(\vx)=p_t(\vx|c)$. The perturbed interpolant $\vx_t=(1-t)\tilde{\vx}_0+t\vx_1+\gamma_t\vz$ has time derivative $\vx_1-\tilde{\vx}_0+\dot{\gamma}_t\vz$, so the minimizer of $\gL_{\vv}$ is $\vv^\star(\vx,t)=\mathbb{E}[\vx_1-\tilde{\vx}_0+\dot{\gamma}_t\vz\mid\vx_t=\vx]$, the probability-flow velocity of the interpolant, whose marginals satisfy the continuity equation $\partial_t p_t=-\nabla\cdot(\vv^\star p_t)$ \citep[Eq.~2.10]{albergo2023stochastic}. The minimizer of $\gL_{\eta}$ is $\boldsymbol{\eta}^\star(\vx,t)=\mathbb{E}[\vz\mid\vx_t=\vx]$, and $\nabla\log p_t=-\boldsymbol{\eta}^\star/\gamma_t$ for $\gamma_t>0$ \citep[Eq.~2.21]{albergo2023stochastic}. Consider the SDE $d\vx_t=\vb(\vx_t,t)\,dt+\sigma_t\,d\vW_t$ with $\vb=\vv^\star+\tfrac12\sigma_t^2\nabla\log p_t$, integrated forward from $t=0$. Using $p_t\nabla\log p_t=\nabla p_t$, the right-hand side of its Fokker--Planck equation evaluated at $p_t$ is
\begin{equation}
-\nabla\cdot(\vb\,p_t)+\tfrac12\sigma_t^2\Delta p_t
=-\nabla\cdot(\vv^\star p_t)-\tfrac12\sigma_t^2\Delta p_t+\tfrac12\sigma_t^2\Delta p_t=\partial_t p_t .
\end{equation}
Hence $p_t$ solves the Fokker--Planck equation; since the SDE starts from $p_0$, its marginals equal $p_t$ for all $t\in[0,1]$ and any $\sigma_t$, which gives \eqref{eq:sde_correct} and~\eqref{eq:sde_learned}. With the drift $\vv^\star-\tfrac12\sigma_t^2\nabla\log p_t$, the residual $-\sigma_t^2\Delta p_t$ does not vanish; that drift corresponds to the time-reversed SDE integrated from $t=1$ to $t=0$. Finally, since $\mathbb{E}[\dot{\gamma}_t\vz\mid\vx_t]=\dot{\gamma}_t\boldsymbol{\eta}^\star=-\gamma_t\dot{\gamma}_t\nabla\log p_t$, we have $\vv^\star=\vu-\gamma_t\dot{\gamma}_t\nabla\log p_t$ with $\vu=\mathbb{E}[\vx_1-\tilde{\vx}_0\mid\vx_t]$, so equivalently $\vb=\vu+(\tfrac12\sigma_t^2-\gamma_t\dot{\gamma}_t)\nabla\log p_t$.

\subsection{Implementation Details}
During training, the source image is perturbed as $\tilde{\vx}_0=\vx_0+\boldsymbol{\epsilon}$, $\boldsymbol{\epsilon}\sim\mathcal{N}(\vzero,\mI)$, with probability $0.5$ \citep{su2025three}; the interpolant uses $\gamma_t=a\sin^2(\pi t)$ with $a=1$; the network predicts $\vv_\theta$ and $\boldsymbol{\eta}_\phi$, and the score is $s_\phi=-\boldsymbol{\eta}_\phi/\gamma_t$. At inference, we integrate \eqref{eq:sde_learned} forward from $t=0$ to $t=1$ with $\sigma_t^2=2\sin^2(\pi t)$, setting the diffusion and score terms to zero within $10^{-2}$ of the endpoints, where $\gamma_t\to0$.

\subsection{Sample Variance of Antithetic Trajectories}
\label{app:antithetic_variance}
Let $x_1,\dots,x_K$ ($K=2J$) be one coordinate of the $K$ SDE samples for a fixed posterior draw, each with variance $\sigma^2$, where the two members of each antithetic pair have correlation $\rho$ and different pairs are independent. Then $\mathrm{Var}(\bar{x})=\sigma^2(1+\rho)/K$, the residual Monte Carlo term of the antithetic conditional-mean estimate, and
\begin{equation}
\mathbb{E}\Big[\tfrac{1}{K-1}\textstyle\sum_{k=1}^K(x_k-\bar{x})^2\Big]=\tfrac{1}{K-1}\big(K\sigma^2-K\,\mathrm{Var}(\bar{x})\big)=\sigma^2\Big(1-\frac{\rho}{K-1}\Big).
\end{equation}
For $\rho<0$, the within-draw sample variance used for the aleatoric term is therefore inflated by the factor $1+|\rho|/(K-1)$, i.e., by 17--26\% for $K=4$ and the correlations in Table~\ref{tab:avuq_variance_reduction}. If $\rho$ were the same for all inputs, this would rescale the aleatoric trace by a constant and leave its ranking unchanged; because $\rho$ varies across inputs, we use the aleatoric trace as a ranking score rather than as an unbiased variance estimate (Sec.~\ref{sec:scalable}).

\section{Two-Sided, Label-Free Anomaly Score}
\label{app:two_sided}
The scores in Sec.~\ref{sec:scalable} assume that uncertainty decreases under the shift. To remove this assumption, we use a two-sided rule fixed in advance. For each component, we standardize the log-trace by the mean and standard deviation of the ID inputs and score an input by the absolute $z$-score, so an input is flagged when its uncertainty is unusual in either direction. ID inputs are scored with 5-fold cross-fitting (statistics computed on the other four folds), and no OOD input, label, or direction is used, so the rule applies unchanged to a new scenario. Table~\ref{tab:two_sided} uses the same per-input scores, inputs, and labels as Table~\ref{tab:ood_detection_nothreshold}.

\begin{table*}[h]
\color{blue}\captionsetup{labelfont={color=blue}}
\centering
\small
\caption{\protect\rev{Two-sided, label-free anomaly score (absolute $z$-score of the log-trace relative to ID inputs). Same per-input scores, inputs, and labels as Table~\ref{tab:ood_detection_nothreshold}; point estimates. Positive rates (AUPR of a random detector) are 0.35 for BBBC021 Unseen Pert.\ and 0.50 for the other cohorts.}}
\resizebox{\textwidth}{!}{
\begin{tabular}{lcccccccccc}
\toprule
& \multicolumn{4}{c}{\textbf{BBBC021}} & \multicolumn{6}{c}{\textbf{JUMP}} \\
\cmidrule(lr){2-5} \cmidrule(lr){6-11}
\textbf{Method} & \multicolumn{2}{c}{Unseen Pert.} & \multicolumn{2}{c}{Intensity Shift} & \multicolumn{2}{c}{Unseen Cell Lines} & \multicolumn{2}{c}{Unseen Plates} & \multicolumn{2}{c}{Intensity Shift}\\
\cmidrule(lr){2-3} \cmidrule(lr){4-5} \cmidrule(lr){6-7} \cmidrule(lr){8-9} \cmidrule(lr){10-11}
& AUROC$\uparrow$ & AUPR$\uparrow$ & AUROC$\uparrow$ & AUPR$\uparrow$ & AUROC$\uparrow$ & AUPR$\uparrow$ & AUROC$\uparrow$ & AUPR$\uparrow$ & AUROC$\uparrow$ & AUPR$\uparrow$\\
\midrule
MAP Aleatoric & 0.532 & 0.363 & 0.552 & 0.547 & 0.526 & 0.534 & 0.514 & 0.516 & 0.581 & 0.505 \\
SWAG Aleatoric & 0.639 & 0.517 & 0.681 & 0.711 & 0.534 & 0.540 & 0.508 & 0.509 & 0.636 & 0.576 \\
SWAG Epistemic & 0.639 & 0.502 & 0.677 & 0.701 & 0.532 & 0.539 & \textbf{0.531} & \textbf{0.523} & 0.624 & 0.566 \\
LA Aleatoric & 0.620 & 0.503 & 0.679 & 0.706 & 0.531 & 0.547 & 0.490 & 0.509 & 0.612 & 0.569 \\
LA Epistemic & 0.583 & 0.400 & 0.651 & 0.632 & 0.520 & 0.520 & 0.519 & 0.521 & 0.356 & 0.396 \\
\midrule
MCD Aleatoric & 0.581 & 0.391 & 0.602 & 0.572 & 0.568 & 0.570 & 0.515 & 0.513 & 0.457 & 0.436 \\
MCD Epistemic & 0.556 & 0.371 & 0.576 & 0.549 & 0.553 & 0.554 & 0.513 & 0.515 & 0.446 & 0.429 \\
\modelname~Aleatoric & 0.554 & 0.381 & 0.585 & 0.569 & \textbf{0.581} & \textbf{0.579} & 0.518 & 0.515 & 0.592 & 0.512 \\
\modelname~Epistemic & \textbf{0.684} & \textbf{0.607} & \textbf{0.689} & \textbf{0.724} & 0.545 & 0.543 & 0.524 & 0.523 & \textbf{0.667} & \textbf{0.586} \\
\bottomrule
\end{tabular}
}
\label{tab:two_sided}
\end{table*}

As expected for a test that does not use the direction, power is lower than with the one-sided score of Table~\ref{tab:ood_detection_nothreshold} (e.g., \modelname~Epistemic 0.764 $\to$ 0.684 on BBBC021 Unseen Pert.). Nevertheless, \modelname~remains the best detector in all four cohorts in which it is best under the one-sided score, with the same component: Epistemic on BBBC021 Unseen Pert., BBBC021 Intensity Shift, and JUMP Intensity Shift, and Aleatoric on JUMP Unseen Cell Lines, where all scores are close. On JUMP Intensity Shift, where aleatoric uncertainty increases under the shift, the \modelname~Aleatoric score moves from below chance (0.224) to 0.592. On JUMP Unseen Plates, all two-sided scores remain near chance (0.49--0.53), consistent with the one-sided results.

\section{Pairing-Free Failure Labels on BBBC021}
\label{app:pairfree}
The filter of Appendix~\ref{app:implementation} measures the error of a generation against one randomly paired target image. On BBBC021, the pretrained MoA classifier of Appendix~\ref{app:moa_classification} provides an error that needs no paired target image: for each input, we classify the model's own four generations (MAP sampling) and define the error as one minus the classifier's probability for the compound's annotated MoA, averaged over the four generations. JUMP has no comparable MoA labels for its perturbations, so this label is available only for BBBC021.

Table~\ref{tab:pairfree_filter} replaces the random-pair distance by this error and otherwise follows Table~\ref{tab:unreliable_generation_detection} (same cohorts, methods, scores, and $\mu\mp0.5\sigma$ rule), on the inputs shared by all methods. \modelname~Epistemic is the best detector in all four entries.

\begin{table}[h]
\color{blue}\captionsetup{labelfont={color=blue}}
\centering
\small
\caption{\protect\rev{Unreliable generation detection on BBBC021 with the pairing-free MoA-classifier error used for filtering. Mean $\pm$ standard error over replicates (three for MCD and \modelname, two for the other methods). Positive rates (AUPR of a random detector) are 0.39--0.46 (Unseen Pert.) and 0.48 (Intensity Shift), depending on the model seed.}}
\begin{tabular}{lcccc}
\toprule
& \multicolumn{2}{c}{BBBC021 Unseen Pert.} & \multicolumn{2}{c}{BBBC021 Intensity Shift}\\
\cmidrule(lr){2-3} \cmidrule(lr){4-5}
\textbf{Method} & AUROC$\uparrow$ & AUPR$\uparrow$ & AUROC$\uparrow$ & AUPR$\uparrow$\\
\midrule
\textbf{non-nested} & & & & \\
MAP Aleatoric & 0.709$\pm$0.058 & 0.633$\pm$0.100 & 0.717$\pm$0.096 & 0.701$\pm$0.075 \\
MCD-DFM Epistemic & 0.679$\pm$0.051 & 0.516$\pm$0.090 & 0.674$\pm$0.065 & 0.609$\pm$0.082 \\
\midrule
\textbf{nested-i.i.d.} & & & & \\
SWAG Aleatoric & 0.706$\pm$0.050 & 0.691$\pm$0.079 & 0.694$\pm$0.087 & 0.738$\pm$0.070 \\
SWAG Epistemic & 0.723$\pm$0.048 & 0.703$\pm$0.093 & 0.687$\pm$0.090 & 0.729$\pm$0.081 \\
LA Aleatoric & 0.693$\pm$0.048 & 0.663$\pm$0.072 & 0.672$\pm$0.108 & 0.702$\pm$0.091 \\
LA Epistemic & 0.692$\pm$0.022 & 0.658$\pm$0.064 & 0.698$\pm$0.108 & 0.728$\pm$0.081 \\
MCD Aleatoric & 0.739$\pm$0.021 & 0.687$\pm$0.050 & 0.701$\pm$0.042 & 0.726$\pm$0.037 \\
MCD Epistemic & 0.726$\pm$0.024 & 0.667$\pm$0.044 & 0.693$\pm$0.040 & 0.714$\pm$0.036 \\
\midrule
\textbf{nested-antithetic} & & & & \\
\modelname~Aleatoric & 0.717$\pm$0.023 & 0.678$\pm$0.047 & 0.713$\pm$0.049 & 0.737$\pm$0.039 \\
\modelname~Epistemic & \textbf{0.788$\pm$0.015} & \textbf{0.786$\pm$0.039} & \textbf{0.733$\pm$0.045} & \textbf{0.778$\pm$0.044} \\
\bottomrule
\end{tabular}
\label{tab:pairfree_filter}
\end{table}
}

\end{document}